\documentclass{article}

\PassOptionsToPackage{numbers, sort&compress}{natbib}
\usepackage[preprint]{neurips_2026}

\usepackage[utf8]{inputenc} 
\usepackage[T1]{fontenc}    
\usepackage{hyperref}       
\usepackage{url}            
\usepackage{booktabs}       
\usepackage{amsfonts}       
\usepackage{nicefrac}       
\usepackage{microtype}      
\usepackage{xcolor}         

\usepackage{amsmath}
\usepackage{bm}
\usepackage{algorithm}
\usepackage{algpseudocode} 
\usepackage{graphicx}
\usepackage{subcaption}
\usepackage{amsthm}
\usepackage{multirow}
\usepackage{booktabs}
\usepackage{placeins}

\newcommand{\ms}[2]{\ensuremath{#1_{\scriptscriptstyle \pm #2}}}
\newcommand{\bestms}[2]{\ensuremath{\mathbf{#1}_{\scriptscriptstyle \pm #2}}}

\newtheorem{theorem}{Theorem}[section]

\input{boldletters.tex}
\title{Tempered Guided Diffusion}

\author{%
  Andreas Makris
    \\
  Department of Mathematics and Statistics \\
  Lancaster University, UK \\
  \texttt{a.makris@lancaster.ac.uk} \\
  \And
  Paul Fearnhead \\
  Department of Mathematics and Statistics \\
  Lancaster University, UK \\
  \texttt{p.fearnhead@lancaster.ac.uk} \\
  \And
  Christopher Nemeth \\
  Department of Mathematics and Statistics \\
  Lancaster University, UK \\
  \texttt{c.nemeth@lancaster.ac.uk} \\
}

\begin{document}

\maketitle

\begin{abstract}
Training-free conditional diffusion provides a flexible alternative to task-specific conditional model training, but existing samplers often allocate computation inefficiently: independent guided trajectories can vary widely in quality, and additional function evaluations along a single trajectory may not recover from poor early decisions. We propose \textbf{Tempered Guided Diffusion (TGD)}, an annealed sequential Monte Carlo framework for training-free conditional sampling with diffusion priors. TGD targets tempered posterior distributions over the clean signal, using noisy diffusion states only as auxiliary variables for proposing reconstructions and propagating particles. Particles are reweighted by incremental likelihood ratios, resampled, and propagated across noise levels, concentrating computation on trajectories plausible under both the prior and observation. Under idealized exact-reconstruction assumptions, full TGD yields a consistent particle approximation to the posterior as the number of particles grows. For expensive reconstruction tasks, \textbf{Accelerated TGD (A-TGD)} retains early particle exploration but prunes to a single high-likelihood trajectory partway through sampling. TGD improves posterior approximation in controlled inverse problems with known posteriors, while A-TGD achieves favorable wall-clock speed-quality tradeoffs for single reconstructions on image inverse problems relative to independent multi-trajectory baselines. \footnote{Code is available at
\url{https://github.com/zedt123/tgd_2026}.}

\end{abstract}

\section{Introduction}
\label{sec:introduction}

\begin{figure}[t]
    \centering
    \includegraphics[width=0.82\linewidth]{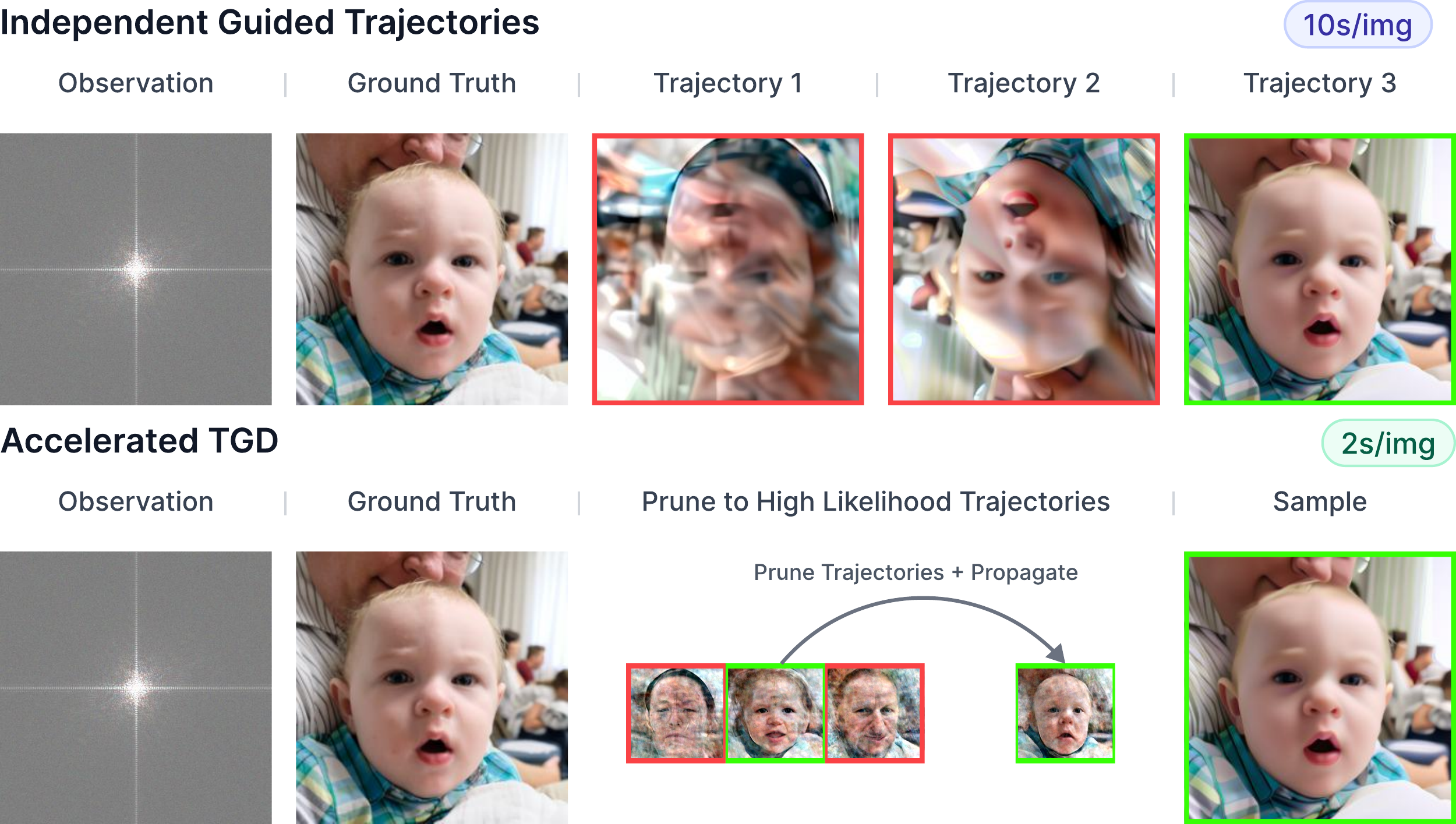}
    \caption{\textbf{Motivation for A-TGD.}
    Independent trajectories select only after completion; A-TGD explores early, prunes by likelihood, and completes one trajectory.}
    \label{fig:motivation}
\end{figure}

Diffusion models \cite{sohldickstein2015deep, song2019generative, song2020improved, ho2020denoisingdiffusionprobabilisticmodels, song2021scorebasedgenerativemodelingstochastic, song2021denoising, nichol2021improved, karras2022elucidatingdesignspacediffusionbased} have emerged as powerful generative priors for conditional generation \cite{dhariwal2021diffusionmodelsbeatgans, ho2022classifierfreediffusionguidance, bansal2024universal, autogduidance, pandey2025variational, gao2025rectified} and inverse problems \cite{song2022solving, kawar2021snips, kawar2022denoising, chung2022improving, chung2024diffusionposteriorsamplinggeneral, wang2023zero, zhu2023denoising, mardani2023variational, murata2023gibbsddrm, pgdm, fps_smc, Zhang_2025, ahmed2025solvingdiffusioninverseproblems}. In many applications, however, the conditioning mechanism changes from task to task, while paired training data for each new setting may be unavailable or costly to obtain. This makes training-free conditional diffusion \cite{he2023manifoldpreservingguideddiffusion, wu2024practicalasymptoticallyexactconditional, tfg, trippe2023diffusion, free_hunch, bansal2024universal} especially attractive: rather than training a separate conditional model for every inverse problem, one reuses a pretrained unconditional diffusion prior and incorporates the observation only at sampling time. This flexibility is useful both when the goal is posterior sampling or to obtain a single high-quality reconstruction.

Despite this appeal, existing training-free conditional diffusion methods can be unreliable on challenging inverse problems. For a fixed observation, different random initializations may produce markedly different reconstructions, some of which remain inconsistent with the observation or have poor perceptual quality \cite{Zhang_2025}. A common response is to run the sampler multiple times and retain the best output according to a data-consistency score \cite{chung2024diffusionposteriorsamplinggeneral, Zhang_2025}. Such best-of-$N$ heuristics can improve robustness, but they treat trajectories independently and spend computation uniformly across candidates, even after some have become unlikely to satisfy the observation. This suggests that the main difficulty is not only the amount of computation, but how it is allocated across candidate trajectories. This perspective is consistent with a broader theme in scalable Monte Carlo methodology \citep{fearnhead2025scalable}: when a single large computation is inefficient, one can often gain robustness by splitting computation across multiple partial simulations and then using weighting, selection, transformation, or recombination to concentrate effort on the statistically relevant components \citep{scott2022bayes,vyner2023swiss}. In TGD, this principle is instantiated through interacting particles rather than through independent best-of-$N$ trajectories.

Motivated by this perspective, we introduce \emph{Tempered Guided Diffusion} (TGD), a sequential Monte Carlo (SMC) approach that maintains a population of candidate reconstructions and uses resampling to focus subsequent computation on the most promising candidates. Existing diffusion-SMC methods have applied particle-based ideas to conditional sampling and inverse problems~\citep{trippe2023diffusion,fps_smc,wu2024practicalasymptoticallyexactconditional}. TGD differs in where the annealed targets are defined: rather than treating noisy diffusion states as the main objects of inference, TGD defines a sequence of tempered posterior targets \citep{delmoral2006sequential, smc_elements, neal1996sampling, neal1998annealedimportancesampling} over the clean signal. For an inverse problem with clean signal $\bx_0 \in \mathbb{R}^{d_x}$, observation $\by \in \mathbb{R}^{d_y}$, and likelihood $p(\by \mid \bx_0)$, these targets are
\begin{equation}
\label{eq:tempered-posterior}
\pi_r(\bx_0) \propto p(\bx_0)\,p(\by \mid \bx_0)^{\lambda_r},
\qquad
0\leq\lambda_R \le \lambda_{R-1} \le \cdots \le \lambda_0 = 1,
\end{equation}
where $R$ is the number of restarts, $\lambda_r \in [0,1]$ controls the strength of the observation, the index $r$ decreases as conditioning increases, and $\pi_0 = p(\bx_0 \mid \by)$ is the full posterior. Noisy diffusion states are used only as auxiliary variables for reconstruction and propagation. At each stage, a training-free diffusion solver, such as an MPGD- or DAPS-style reconstruction module~\citep{he2023manifoldpreservingguideddiffusion,Zhang_2025}, maps each noisy auxiliary particle to a candidate clean reconstruction. The candidates are then reweighted by the incremental likelihood ratio needed to move from $\pi_r$ to $\pi_{r-1}$; resampling discards low-likelihood candidates and duplicates promising ones before the retained particles are re-noised for the next stage. Thus, unlike independent best-of-$N$ sampling, TGD uses multiple reconstructions as an interacting population, concentrating later computation on candidates that remain plausible under both the prior and the observation.

For expensive image inverse problems, we use \emph{Accelerated TGD} (A-TGD), which prunes the particle population once at an intermediate stage (see Figure~\ref{fig:motivation}) and completes the same schedule with a single high-likelihood trajectory.

\paragraph{Our contributions:}
(i) We introduce TGD, an annealed SMC framework for training-free conditional sampling over clean signals, and show that several existing training-free samplers can be viewed as different stagewise reconstruction modules within the same outer particle scheme.
(ii) Under exact stagewise reconstruction and standard SMC assumptions, we show that idealized TGD is asymptotically consistent for the posterior as the number of particles grows.
(iii) We develop Accelerated TGD (A-TGD), a pruning-based variant that preserves early particle exploration while reducing the cost of reconstruction.
(iv) We evaluate TGD in controlled inverse problems with known posteriors and on image inverse problems including inpainting and phase retrieval, showing improved posterior approximation in the controlled settings and that A-TGD improves wall-clock speed-quality tradeoffs for image reconstruction.

\section{Preliminaries}
\label{sec:preliminaries}

\paragraph{Inverse problems.}
We consider inverse problems in which an unknown clean signal $\bx_0 \in \mathbb{R}^{d_x}$ is observed indirectly through a measurement $\by \in \mathbb{R}^{d_y}$. The likelihood $p(\by \mid \bx_0)$ encodes the probabilistic relationship between the clean signal and the observation, including the measurement process, or forward operator, and observation noise. Given a prior $p(\bx_0)$ over clean signals, the posterior \citep{gelman1995bayesian,murphy2022probabilistic,fearnhead2025scalable} is
\begin{equation}
\label{eq:posterior}
    p(\bx_0 \mid \by) \propto p(\bx_0)\,p(\by \mid \bx_0).
\end{equation}
We represent the prior $p(\bx_0)$ using a pretrained diffusion model. We adopt a standard continuous-time formulation based on a forward noising process and its corresponding reverse dynamics.

\paragraph{Unconditional diffusion.}
Let $(\bx_s)_{s \in [0,S]}$ denote a forward diffusion process defined by
\[
\mathrm{d}\bx_s = g(s)\, \mathrm{d}\bw_s,
\qquad \bx_0 \sim p(\bx_0),
\]
where \(\bw_s\) is standard Brownian motion and \(g(s)\) is a scalar noise schedule. This defines a family of marginal densities \(p_s(\bx_s)\), where \(p_s\) is the distribution obtained by sampling \(\bx_0 \sim p(\bx_0)\) and evolving the forward process to time \(s\). For sufficiently large \(S\), \(p_S\) is well-approximated by a simple Gaussian distribution \cite{karras2022elucidatingdesignspacediffusionbased, song2021scorebasedgenerativemodelingstochastic}.

Under standard regularity conditions~\citep{ANDERSON1982313,song2021scorebasedgenerativemodelingstochastic}, the reverse-time dynamics that transform $\bx_S \sim p_S$ back to $\bx_0 \sim p(\bx_0)$ are given by the reverse SDE
\begin{equation}
\mathrm{d}\bx_s =
\left[-g(s)^2 \nabla_{\bx_s}\log p_s(\bx_s)\right] \mathrm{d}s
+ g(s)\mathrm{d}\bar{\bw}_s,
\label{eq:reverse-sde}
\end{equation}
where $\bar{\bw}_s$ denotes reverse-time Brownian motion. One can also use the corresponding probability-flow ODE \citep{song2021scorebasedgenerativemodelingstochastic},
\begin{equation}
\mathrm{d}\bx_s =
-\frac{1}{2}g(s)^2 \nabla_{\bx_s}\log p_s(\bx_s)\,\mathrm{d}s,
\label{eq:probability-flow-ode}
\end{equation}
which has the same time-marginal distributions as the reverse SDE when integrated from $S$ to $0$. In practice, the score function $\nabla_{\bx_s} \log p_s(\bx_s)$ is approximated by a neural network trained on noisy samples \cite{karras2022elucidatingdesignspacediffusionbased,song2021scorebasedgenerativemodelingstochastic}. In the \emph{Elucidated Diffusion Model} (EDM) parameterization \cite{karras2022elucidatingdesignspacediffusionbased}, this corresponds to predicting a denoised estimate $\hat{\bx}_0(\bx_s,s)$, yielding
\[
\nabla_{\bx_s} \log p_s(\bx_s)
=
\frac{\hat{\bx}_0(\bx_s,s) - \bx_s}{\sigma_s^2},
\]
where $\sigma_s^2 = \int_0^s g(t)^2\,\mathrm{d}t$ is the accumulated noise variance of the forward process.

\paragraph{Training-free conditional sampling.}
Given an observation $\by$, we wish to sample from the posterior \eqref{eq:posterior} without retraining the diffusion model. Training-free conditional diffusion uses the property that, if we adapt the reverse dynamics by replacing the unconditional score in either the reverse SDE~\eqref{eq:reverse-sde} or the probability-flow ODE~\eqref{eq:probability-flow-ode} with the conditional score, then simulating the resulting conditional dynamics to time $0$ would yield samples from the posterior \eqref{eq:posterior}. The conditional score is defined using the conditional-score decomposition
\begin{equation}
\label{eq:conditional-score}
\nabla_{\bx_s} \log p_s(\bx_s \mid \by)
=
\nabla_{\bx_s} \log p_s(\bx_s)
+
\nabla_{\bx_s} \log p_s(\by \mid \bx_s),
\end{equation}
where
\[
p_s(\by \mid \bx_s)
=
\int p(\by \mid \bx_0)\, p(\bx_0 \mid \bx_s)\, \mathrm{d}\bx_0
\]
is the observation likelihood marginalized through the diffusion conditional $p(\bx_0 \mid \bx_s)$.

Incorporating \eqref{eq:conditional-score} into the reverse SDE yields condition-aware dynamics that map a noisy state $\bx_s$ toward samples consistent with $\by$. In practice, however, the term $\nabla_{\bx_s} \log p_s(\by \mid \bx_s)$ is intractable and must be approximated. Existing methods differ in this approximation: manifold-preserving guided diffusion (MPGD) \cite{he2023manifoldpreservingguideddiffusion} and related approaches \cite{chung2024diffusionposteriorsamplinggeneral, Zhang_2025, free_hunch} construct approximate gradients in data or measurement space, while others rely on plug-in likelihood approximations or score corrections.

Another common strategy is to introduce \emph{restarts}, repeatedly reinitializing the diffusion process at intermediate noise levels and reapplying conditional updates~\cite{Zhang_2025, ahmed2025solvingdiffusioninverseproblems, yu2023freedom, repaint, bansal2024universal, tfg}. DAPS, for example, alternates noising and denoising steps and can use a coarse ODE discretization to obtain efficient reconstructions between restart levels~\citep{Zhang_2025}. Restarts can improve robustness, but they are typically used as an ad hoc sampling heuristic. In TGD, these short conditional reconstruction procedures are used as modules that map noisy auxiliary states to candidate clean signals; the outer procedure is instead formulated separately as SMC over an annealed sequence of clean-space posteriors.

\section{Tempered Guided Diffusion}
\label{sec:tempered-guided-diffusion}

\subsection{Annealed posterior paths over clean signals}
\label{sec:annealed-clean-targets}

Our goal is to sample from the posterior $p(\bx_0 \mid \by)$. Rather than imposing the full observation likelihood at once, TGD introduces an annealed sequence~\citep{delmoral2006sequential,smc_elements,cabezas2024markovian} of tempered targets \citep{neal1996sampling,nemeth2019pseudo} over the clean signal $\bx_0$,
\begin{equation}
\label{eq:tgd-path}
\pi_r(\bx_0)
\propto
p(\bx_0)\,p(\by \mid \bx_0)^{\lambda_r},
\qquad
0\leq\lambda_R \le \lambda_{R-1} \le \cdots \le \lambda_0 = 1,
\end{equation}
where $R$ is the number of restarts, $\lambda_r \in [0,1]$ is the tempering parameter at stage $r$. The index $r$ decreases as conditioning increases: $\pi_0 = p(\bx_0 \mid \by)$ is the full posterior, while earlier stages impose the observation more weakly. When $\lambda_R=0$, the initial target is the prior. Practical choices of the tempering schedule
are described in Appendix~\ref{app:ann_schedules}.

The key modeling choice in TGD is that the annealed path \eqref{eq:tgd-path} is defined over clean signals rather than noisy diffusion states. We refer to each $\pi_r$ as a clean-space target because it is a distribution over the unknown clean signal $\bx_0$. An outer stage pairs a tempering level $\lambda_r$ with a diffusion noise level $s_r$. At stage $r$, the particle system is represented by noisy auxiliary states $\{\bz_r^{(i)}\}_{i=1}^N$ at noise level $s_r$. These states are not themselves the inferential targets; they are diffusion-compatible variables used to propose and propagate clean candidates.

Given a noisy auxiliary state $\bz_r^{(i)}$, a reconstruction step produces a candidate clean signal $\bx_0^{(i)}$. These clean candidates are reweighted by the incremental likelihood ratio needed to move from $\pi_r$ to $\pi_{r-1}$ and may then be resampled. The retained clean particles are subsequently re-noised to form auxiliary states at the next noise level. Thus, the sequence $\{\pi_r\}_{r=0}^R$ defines the clean-space inference path, while diffusion supplies the reconstruction and re-noising operations
used to move the particle population along that path.

We write $p_s(\bz \mid \bx_0)$ for the forward noising kernel and $p_s(\bx_0 \mid \bz)\propto p(\bx_0)p_s(\bz \mid \bx_0)$ for the corresponding conditional density.

\subsection{Stagewise construction of TGD}
\label{sec:stagewise}

We approximate the targets $\{\pi_r\}_{r=0}^R$ with a weighted particle system evolved by SMC \cite{doucet2001sequential, delmoral2004feynman, chopin2004introduction, smc_elements}. At outer stage $r$, particles are represented by auxiliary states $\{(\bz_r^{(i)}, w_r^{(i)})\}_{i=1}^N$ at noise level $s_r$, where
\[
S = s_R \ge s_{R-1} \ge \cdots \ge s_0 > 0.
\]

We initialize from the noisy prior,
\[
\bz_R^{(i)} \sim p_{s_R}(\bz_R),
\qquad
w_R^{(i)} = \frac{1}{N},
\qquad i=1,\dots,N.
\]
When $\lambda_R=0$, this is exactly the noisy marginal corresponding to the initial clean-space target. More generally, it should be viewed as a prior initialization before the first tempered reconstruction step.

For $r=R,\dots,1$, one TGD update from stage $r$ to stage $r-1$ consists of: (i) reconstruction, (ii) weighting/resampling, and (iii) propagation.

\paragraph{Reconstruction.}
Given $\bz_r^{(i)}$, we sample a clean candidate from the stagewise reconstruction law
\[
\bx_0^{(i)} \sim p_{\lambda_r,s_r}(\bx_0 \mid \bz_r^{(i)}, \by),
\]
where
\begin{equation}
\label{eq:reconstruction}
p_{\lambda_r,s_r}(\bx_0 \mid \bz_r, \by)
\propto
p_{s_r}(\bx_0 \mid \bz_r)\,p(\by \mid \bx_0)^{\lambda_r}
\propto
p(\bx_0)\,p_{s_r}(\bz_r \mid \bx_0)\,p(\by \mid \bx_0)^{\lambda_r}.
\end{equation}

This maps each noisy auxiliary particle to a clean reconstruction under the current tempering level.

\paragraph{Weighting and resampling.}
To move from $\pi_r$ to $\pi_{r-1}$, we update weights using the incremental likelihood ratio
\[
\tilde w_{r-1}^{(i)}
=
w_r^{(i)} \, p(\by \mid \bx_0^{(i)})^{\lambda_{r-1} - \lambda_r}.
\]
After normalization, particles whose clean reconstructions better satisfy the observation receive larger weight as conditioning is strengthened. This likelihood-ratio form follows from the lifted construction underlying TGD (see Appendix~\ref{app:smc-weight-derivations}). To control weight degeneracy, one may resample particles using a standard
low-variance scheme. In our experiments, when resampling is enabled, we use
systematic resampling \citep{smc_elements} and reset all weights to \(1/N\).

\paragraph{Propagation.}
Each retained clean particle is propagated to the next outer stage by re-noising,
\[
\bz_{r-1}^{(i)} \sim p_{s_{r-1}}(\bz_{r-1} \mid \bx_0^{(i)}),
\]
yielding a new population of noisy auxiliary states at level $s_{r-1}$.

After the outer loop, a final reconstruction at $r=0$ produces clean samples
\[
\bx_0^{(i)} \sim p_{\lambda_0,s_0}(\bx_0 \mid \bz_0^{(i)}, \by),
\]
and TGD returns the weighted particle set $\{(\bx_0^{(i)}, w_0^{(i)})\}_{i=1}^N$. 

\begin{algorithm}[t]
\caption{Tempered Guided Diffusion (TGD)}
\label{alg:tgd}
\begin{algorithmic}[1]
\Require observation $\by$, number of particles $N$, tempering schedule $\{\lambda_r\}_{r=0}^R$, noise schedule $\{s_r\}_{r=0}^R$, conditional reconstruction law $p_{\lambda_r,s_r}(\bx_0 \mid \bz_r,\by)$
\State Sample $\bz_R^{(i)} \sim p_{s_R}(\bz_R)$ and set $w_R^{(i)} \gets 1/N$, for $i=1,\dots,N$

\For{$r = R, R-1, \dots, 1$}
    \For{$i = 1,\dots,N$}
        \State Sample $\bx_0^{(i)} \sim p_{\lambda_r,s_r}(\bx_0 \mid \bz_r^{(i)}, \by)$
        \State $\tilde w_{r-1}^{(i)} \gets w_r^{(i)} \, p(\by \mid \bx_0^{(i)})^{\lambda_{r-1}-\lambda_r}$
    \EndFor
    \State $w_{r-1}^{(i)} \gets \tilde w_{r-1}^{(i)} \big/ \sum_{j=1}^N \tilde w_{r-1}^{(j)}$, for $i=1,\dots,N$
    \State Optionally resample according to $\{w_{r-1}^{(i)}\}_{i=1}^N$ and reset weights to $1/N$
    \For{$i = 1,\dots,N$}
        \State Sample $\bz_{r-1}^{(i)} \sim p_{s_{r-1}}(\bz_{r-1}\mid \bx_0^{(i)})$
    \EndFor
\EndFor

\For{$i = 1,\dots,N$}
    \State Sample $\bx_0^{(i)} \sim p_{\lambda_0,s_0}(\bx_0 \mid \bz_0^{(i)}, \by)$
\EndFor

\State \Return $\{(\bx_0^{(i)}, w_0^{(i)})\}_{i=1}^N$
\end{algorithmic}
\end{algorithm}

\subsection{Practical reconstruction modules}
\label{sec:reconstruction-modules}

The ideal reconstruction law \eqref{eq:reconstruction} is generally unavailable, so practical TGD replaces this draw with an approximate training-free conditional diffusion solver. We retain the ideal incremental likelihood potential, but the solver's induced reconstruction density is neither evaluated nor included in an importance correction. Practical TGD therefore approximates the ideal Feynman--Kac flow and is not an exactly weighted SMC sampler for the stated targets. The exact weight derivation and consistency result apply only under exact stagewise reconstruction.

We instantiate this interface with MPGD-style and DAPS-style modules. MPGD/DPS-style methods approximate the conditional score term in \eqref{eq:conditional-score}, while DAPS-style methods combine unconditional diffusion reconstruction with clean-space observation correction. In all image experiments, each reconstruction module uses a coarse four-step ODE discretization between restart levels, so cost scales with the number of particles, outer stages, and reconstruction steps per stage.

This modular view treats existing training-free solvers as approximations to the same stagewise reconstruction law; detailed reductions are deferred to Appendices~\ref{app:hybrid-conditional} and~\ref{app:recovering-existing-training-free-conditional-samplers}.

\subsection{Idealized consistency of TGD}
\label{sec:idealized-exactness-tgd}

We analyze an ideal reference version of TGD to isolate the particle approximation from reconstruction error. In this version, each reconstruction step samples exactly from
\[
p_{\lambda_r,s_r}(\bx_0 \mid \bz_r,\by)
\propto
p_{s_r}(\bx_0 \mid \bz_r)\,p(\by \mid \bx_0)^{\lambda_r},
\]
and the full particle population is propagated without pruning. Exact reconstruction maps the noisy auxiliary marginal associated with $\pi_r$ back to the clean target $\pi_r$; likelihood tempering then reweights particles from $\pi_r$ to $\pi_{r-1}$, and re-noising produces the next auxiliary marginal. Thus ideal TGD is a lifted SMC sampler \citep{delmoral2006sequential} over the clean-space targets.

\begin{theorem}[Consistency of ideal TGD]
Assume $\lambda_R=0$, so that initialization from the noisy prior matches the initial clean-space target. Suppose further that the stagewise reconstruction kernel is exact at each outer stage, the incremental potentials
\[
G_r(\bx_0)=p(\by \mid \bx_0)^{\lambda_{r-1}-\lambda_r}
\]
are bounded, the initial particles are sampled i.i.d. from the noisy prior, and resampling, when used, is performed with a standard SMC scheme \cite{smc_elements, chopin2004introduction}. Then, for every bounded measurable test function $h$,
\[
\sum_{i=1}^N w_0^{(i)} h(\bx_0^{(i)})
\;\xrightarrow[]{\mathbb{P}}\;
\mathbb{E}_{p(\bx_0 \mid \by)}[h(\bx_0)].
\]
Equivalently, $\hat{\pi}_0^N$ converges weakly in probability to the posterior $p(\bx_0 \mid \by)$.
\end{theorem}

The formal assumptions and proof are given in Appendix~\ref{app:consistency-tgd}. This result applies only to the ideal full-particle procedure; practical reconstruction modules and the pruning step in A-TGD introduce additional approximation.

\subsection{Accelerated TGD}
\label{sec:accelerated-tgd}

TGD is designed to approximate the posterior by propagating a particle population through all outer stages. For image inverse problems, where the goal is often a single reconstruction under a fixed compute budget, we use \emph{Accelerated TGD} (A-TGD): a pruning-based variant that runs the full particle procedure only during the early stages.

A-TGD starts with $N$ particles and follows TGD until a pruning fraction $\rho \in [0,1]$ of the outer stages has been completed. At this point, it reconstructs clean candidates $\{\bx_0^{(i)}\}_{i=1}^N$ using the current reconstruction module and retains the particle with largest likelihood,
\[
i^\star
=
\arg\max_{i \in \{1,\dots,N\}}
p(\by \mid \bx_0^{(i)}).
\]
All other particles are discarded, and the auxiliary state corresponding to the selected particle is continued through the remaining stages with a single trajectory. Thus A-TGD keeps the exploratory benefit of multiple early particles while avoiding the cost of propagating the full population to $r=0$.

The pruning fraction $\rho$ controls the trade-off between exploration and computation: smaller values prune earlier and are cheaper, while larger values more closely track full TGD. Because pruning collapses the particle population, A-TGD is a reconstruction-oriented acceleration and is not covered by the posterior-consistency result of Section~\ref{sec:idealized-exactness-tgd}. 

\section{Related Work}
\label{sec:related-work}

Training-free diffusion inverse solvers use a pretrained unconditional diffusion model and incorporate the observation only at sampling time. DPS \cite{chung2024diffusionposteriorsamplinggeneral}, MPGD \cite{he2023manifoldpreservingguideddiffusion}, DAPS \cite{Zhang_2025}, and related approaches \cite{free_hunch, ahmed2025solvingdiffusioninverseproblems} differ in how they impose data consistency, for example through approximate conditional guidance, manifold constraints, decoupled posterior updates, or repeated corrections at intermediate noise levels. RePaint \cite{repaint} and FreeDoM \cite{yu2023freedom} further show that revisiting intermediate noise levels can improve robustness. TGD is complementary to these approaches: it treats conditional diffusion solvers as stagewise reconstruction modules inside an outer annealed particle procedure, rather than prescribing a single guidance approximation. The additional particle layer lets intermediate reconstructions interact through likelihood weighting and resampling, reallocating later computation toward candidates that better satisfy the observation. We make these links explicit in Appendix~\ref{app:recovering-existing-training-free-conditional-samplers}, where several training-free samplers are recovered as special cases of the reconstruction interface.

TGD is also related to particle-based approaches for diffusion models and inverse problems, including SMCDiff~\citep{trippe2023diffusion}, TDS~\citep{wu2024practicalasymptoticallyexactconditional}, and FPS-SMC~\citep{fps_smc}. These methods demonstrate the usefulness of SMC ideas for diffusion-based sampling. More broadly, particle systems have also been used outside diffusion models for latent-variable inference and learning, including marginal maximum-likelihood training methods in which particles approximate intractable posterior or marginal quantities \citep{kuntz2023particle,sharrock2024tuning}. TGD differs in where the annealed targets and particle weights are defined. Its SMC targets are clean-space distributions over $\bx_0$, so each reconstructed candidate is weighted using the observation likelihood $p(\by \mid \bx_0)$ evaluated on a clean signal. Noisy diffusion states are instead auxiliary variables: after resampling, retained clean particles are re-noised to produce the next auxiliary states and then reconstructed again. Thus, diffusion supplies proposal and refresh moves, while likelihood weighting and resampling act directly on clean reconstructions. This also distinguishes TGD from independent best-of-$N$ sampling \cite{Zhang_2025}, where trajectories are run separately and selected only at the end; A-TGD uses this interaction early, then prunes for efficient reconstruction.

\section{Experiments}
\label{sec:experiments}

\paragraph{Setup.}

We evaluate TGD in controlled inverse problems with two- and twenty-dimensional clean signals and known posteriors, and A-TGD on image inverse problems. The controlled problems use noisy elementwise absolute-value observations and measure posterior approximation using max-sliced Wasserstein distance (SWD) \cite{bonneel2015sliced, flamary2021pot, flamary2024pot} to reference posterior samples. For images, we consider inpainting and phase retrieval on FFHQ \cite{karras2019style} and ImageNet \cite{deng2009imagenet}, reporting LPIPS \cite{zhang2018unreasonable}, PSNR, and SSIM averaged over 100 validation images. Note that these image metrics assess single-reconstruction fidelity; they do not establish posterior coverage or calibration. In all image-result tables, entries are reported as mean \(\pm\) standard deviation over the 100 validation images. We compare against DPS, DAPS with one trajectory, and DAPS with four independent trajectories, denoted DAPS $(N=4)$, where the final reconstruction is selected by lowest measurement error. A-TGD also starts with $N=4$ particles but prunes to a single trajectory partway through sampling. Image methods are compared at matched wall-clock runtime within each task and dataset. All image experiments use four Euler ODE steps per reconstruction module (except DPS, see Appendix \ref{app:details}) and a single NVIDIA H200 GPU with 141GB memory.

\subsection{Controlled posterior approximation}
\label{sec:2d-experiments}

\begin{figure*}[t]
    \centering

    \begin{subfigure}[t]{0.59\textwidth}
        \centering
        \includegraphics[height=0.17\textheight]{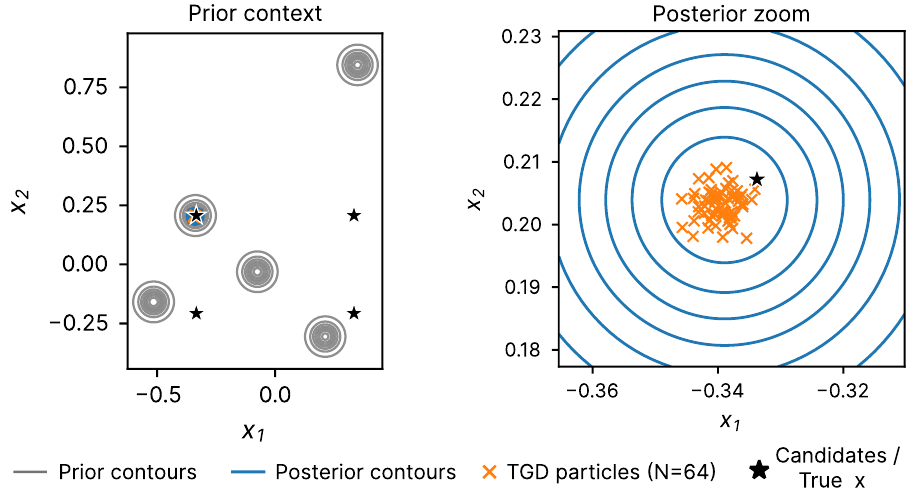}
        \caption{Prior context and posterior zoom.}
        \label{fig:toy2d_contours}
    \end{subfigure}
    \hfill
    \begin{subfigure}[t]{0.39\textwidth}
        \centering
        \includegraphics[height=0.17\textheight]{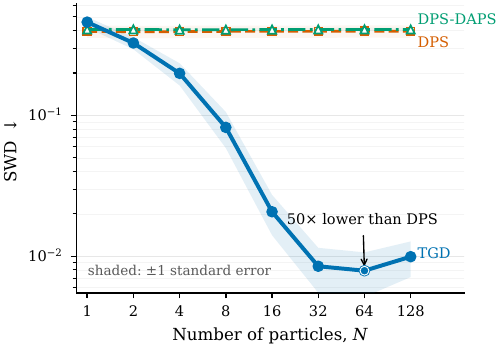}
        \caption{Posterior approximation vs. particle count.}
        \label{fig:toy2d_swd}
    \end{subfigure}

    \caption{
    \textbf{Controlled 2D inverse problem.}
    \textbf{Left:} prior context and posterior zoom for a representative test condition.
    Orange crosses denote the final TGD particles with \(N=64\); black markers denote the
    ground-truth clean sample and the sign-ambiguous candidates induced by the absolute-value
    observation.
    \textbf{Right:} SWD as a function of particle count. Bands denote \(\pm 1\) standard error
    over 10 test conditions. TGD improves rapidly with particle count, while DPS and DPS-DAPS
    do not benefit from additional independent particles in this setting. Lower is better.
    }
    \label{fig:toy2d_results}
\end{figure*}

We first evaluate whether the outer particle procedure improves posterior approximation in a controlled setting where the posterior is known. We consider a two-dimensional inverse problem with a Gaussian-mixture prior and noisy elementwise absolute-value observations,
\[
\by = |\bx_0| + \boldsymbol{\epsilon},
\qquad
\boldsymbol{\epsilon} \sim \mathcal{N}(0,\tau^2 I).
\]
 To isolate the particle inference mechanism from score-model error, we use the ground-truth unconditional score and compare final particles to reference posterior samples using SWD.

Figure~\ref{fig:toy2d_results} compares TGD to baselines and ablations as the number of particles increases. To make the comparison focus on the outer inference procedure, all methods use the DPS reconstruction proposal. TGD combines likelihood tempering with resampling, while the DPS-DAPS ablation removes these outer particle updates. TGD is the only method whose SWD consistently decreases toward zero as the number of particles grows, empirically matching the behavior predicted by the idealized particle construction. 

We also extend this inverse problem to twenty dimensions, with $x_0\in\mathbb{R}^{20}$, and replace the DPS reconstruction proposal with exact draws from~(7). This directly evaluates the ideal procedure in Theorem~3.1. With $N=128$ particles, Exact-TGD achieves an SWD of $1.74\times10^{-4}$ using 1,000 fixed projections, compared with $1.55\times10^{-4}$ between independent exact-posterior sample sets. Appendix~F.1 provides the experimental details, marginal-error measures, and effective-sample-size diagnostics.

\subsection{Image inverse problems at matched runtime}
\label{sec:image-experiments}

Table~\ref{tab:image-results} reports reconstruction quality at matched wall-clock runtime. Across both inpainting datasets, A-TGD achieves the best LPIPS and SSIM, while remaining competitive in PSNR. On FFHQ phase retrieval, where independent guided trajectories vary substantially in quality, A-TGD improves over DAPS $(N=4)$ across all metrics. On ImageNet phase retrieval, the gains are more mixed: A-TGD matches DAPS $(N=4)$ in PSNR, improves SSIM slightly, but does not improve LPIPS at the matched budget. Overall, these results suggest that A-TGD can use a fixed wall-clock budget more effectively than independent best-of-$N$ sampling, especially when early multi-trajectory exploration identifies a high-likelihood trajectory that can be continued alone. Paired significance tests and ablations on particle count, pruning, annealing, resampling, and reconstruction modules are provided in Appendix~\ref{app:ablations}.

\begin{table}[t]
  \caption{Image inverse-problem results at matched wall-clock runtime. DAPS ($N=4$) runs four independent trajectories and selects by measurement error; A-TGD starts with four particles and prunes to one trajectory.}
  \label{tab:image-results}
  \centering
  \small
  \begin{tabular}{lllccc}
    \toprule
    Task & Dataset & Algorithm & PSNR $\uparrow$ & SSIM $\uparrow$ & LPIPS $\downarrow$ \\
    \midrule

    \multirow{8}{*}{Inpainting}
      & \multirow{4}{*}{FFHQ}
        & A-TGD (ours) & \bestms{25.14}{2.54} & \bestms{0.853}{0.025} & \bestms{0.122}{0.028} \\
      & & DAPS ($N=1$) & \ms{24.93}{2.57} & \ms{0.801}{0.028} & \ms{0.153}{0.033} \\
      & & DAPS ($N=4$) & \ms{24.52}{2.25} & \ms{0.790}{0.030} & \ms{0.157}{0.034} \\
      & & DPS & \ms{22.12}{2.51} & \ms{0.806}{0.030} & \ms{0.167}{0.033} \\
      \cmidrule(lr){2-6}
      & \multirow{4}{*}{ImageNet}
        & A-TGD (ours) & \bestms{20.89}{3.67} & \bestms{0.781}{0.037} & \bestms{0.204}{0.036} \\
      & & DAPS ($N=1$) & \ms{20.87}{3.65} & \ms{0.753}{0.035} & \ms{0.218}{0.044} \\
      & & DAPS ($N=4$) & \ms{20.70}{3.42} & \ms{0.740}{0.033} & \ms{0.227}{0.044} \\
      & &  DPS & \ms{18.39}{2.57} & \ms{0.710}{0.069} & \ms{0.281}{0.069} \\

    \midrule

    \multirow{8}{*}{Phase Retrieval}
      & \multirow{4}{*}{FFHQ}
        & A-TGD (ours) & \bestms{29.43}{4.07} & \bestms{0.811}{0.104} & \bestms{0.172}{0.106} \\
      & & DAPS ($N=1$) & \ms{26.98}{8.35} & \ms{0.767}{0.193} & \ms{0.225}{0.185} \\
      & & DAPS ($N=4$) & \ms{27.81}{5.12} & \ms{0.765}{0.128} & \ms{0.207}{0.111} \\
      & &  DPS & \ms{12.65}{2.42} & \ms{0.331}{0.095} & \ms{0.564}{0.063} \\
      \cmidrule(lr){2-6}
      & \multirow{4}{*}{ImageNet}
        & A-TGD (ours) & \ms{22.18}{8.75} & \bestms{0.617}{0.276} & \ms{0.326}{0.218} \\
      & & DAPS ($N=1$) & \ms{19.42}{9.31} & \ms{0.499}{0.298} & \ms{0.413}{0.222} \\
      & & DAPS ($N=4$) & \bestms{22.41}{9.26} & \ms{0.593}{0.291} & \bestms{0.325}{0.221} \\
      & &  DPS & \ms{12.58}{2.48} & \ms{0.232}{0.140} & \ms{0.601}{0.064} \\

    \bottomrule
  \end{tabular}
\end{table}

\subsection{Speed-quality tradeoffs}
\label{sec:speed-quality}

The matched-runtime results in Table~\ref{tab:image-results} compare methods at a fixed budget. We further evaluate how reconstruction quality changes with wall-clock time on FFHQ phase retrieval, where independent trajectories exhibit high variability and best-of-$N$ selection is a strong baseline. Figure~\ref{fig:speed-quality} plots LPIPS as a function of seconds per image for DPS, DAPS $(N=1)$, DAPS $(N=4)$, and A-TGD. DPS remains poor even as the number of function evaluations increases, while DAPS generally improves with runtime, especially for the best-of-four setting, but it must pay the cost of all independent trajectories throughout sampling. A-TGD achieves a better speed-quality curve by using multiple trajectories early and then pruning to a single high-likelihood trajectory for the remaining stages.

Figure~\ref{fig:speed-quality}(b) summarizes this tradeoff by reporting the speedup of A-TGD over DAPS $(N=4)$ at fixed LPIPS thresholds. A-TGD reaches the same perceptual-quality thresholds faster across the evaluated range, suggesting that pruning retains much of the benefit of early multi-trajectory exploration while avoiding the cost of propagating all trajectories to completion.

\begin{figure}[t]
    \centering

    \begin{subfigure}{0.56\linewidth}
        \centering
        \includegraphics[width=\linewidth]{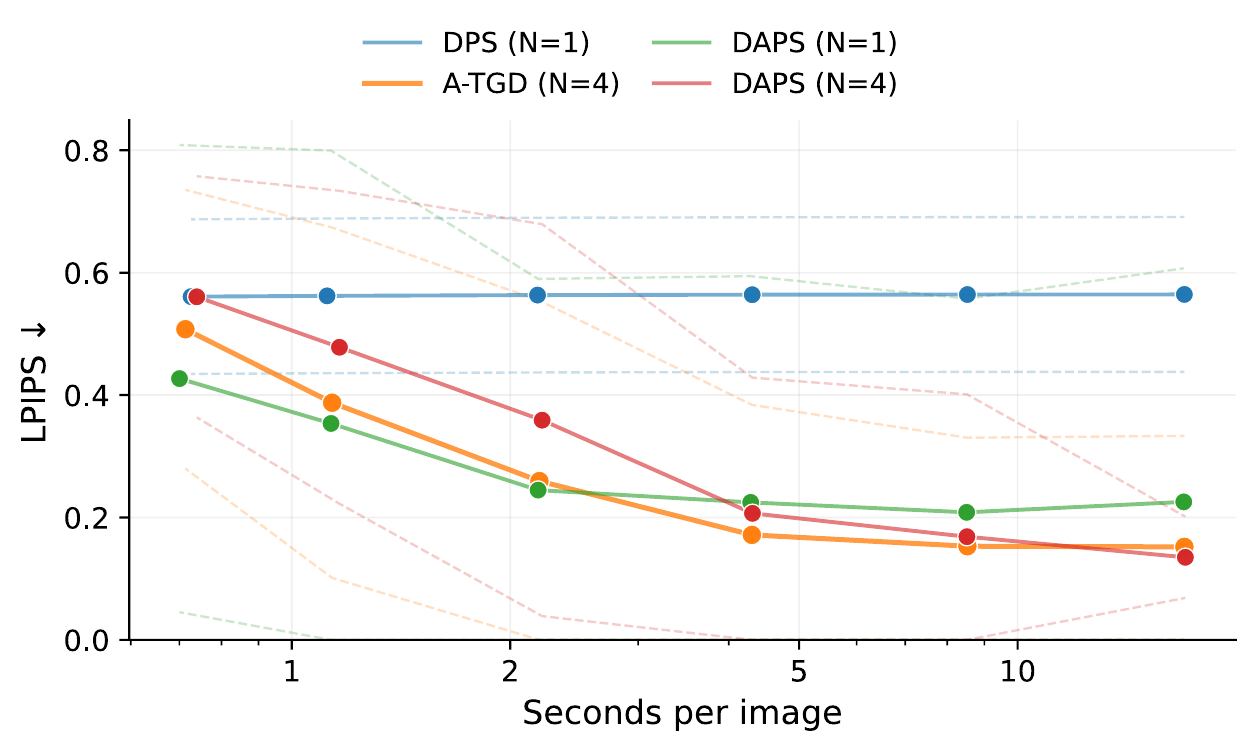}
        \caption{LPIPS vs. wall-clock time}
        \label{fig:ffhq_pr_time_vs_lpips}
    \end{subfigure}
    \hfill
    \begin{subfigure}{0.42\linewidth}
        \centering
        \includegraphics[width=\linewidth]{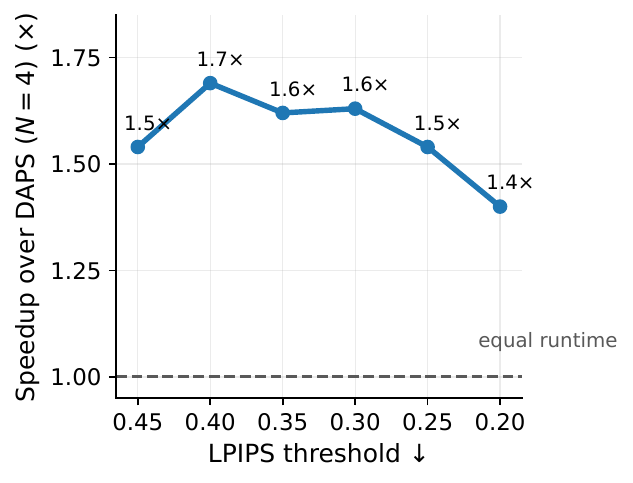}
        \caption{Speedup to LPIPS threshold}
        \label{fig:ffhq_pr_speedup_lpips}
    \end{subfigure}

    \caption{\textbf{Speed-quality tradeoff on FFHQ phase retrieval.}
    \textbf{(a)} LPIPS as a function of wall-clock time per image for DPS, DAPS $(N=1)$, DAPS $(N=4)$, and A-TGD.
    \textbf{(b)} Speedup of A-TGD over DAPS $(N=4)$ to reach fixed LPIPS thresholds.}
    \label{fig:speed-quality}
\end{figure}

\begin{figure}[t]
    \centering
    \includegraphics[width=\linewidth]{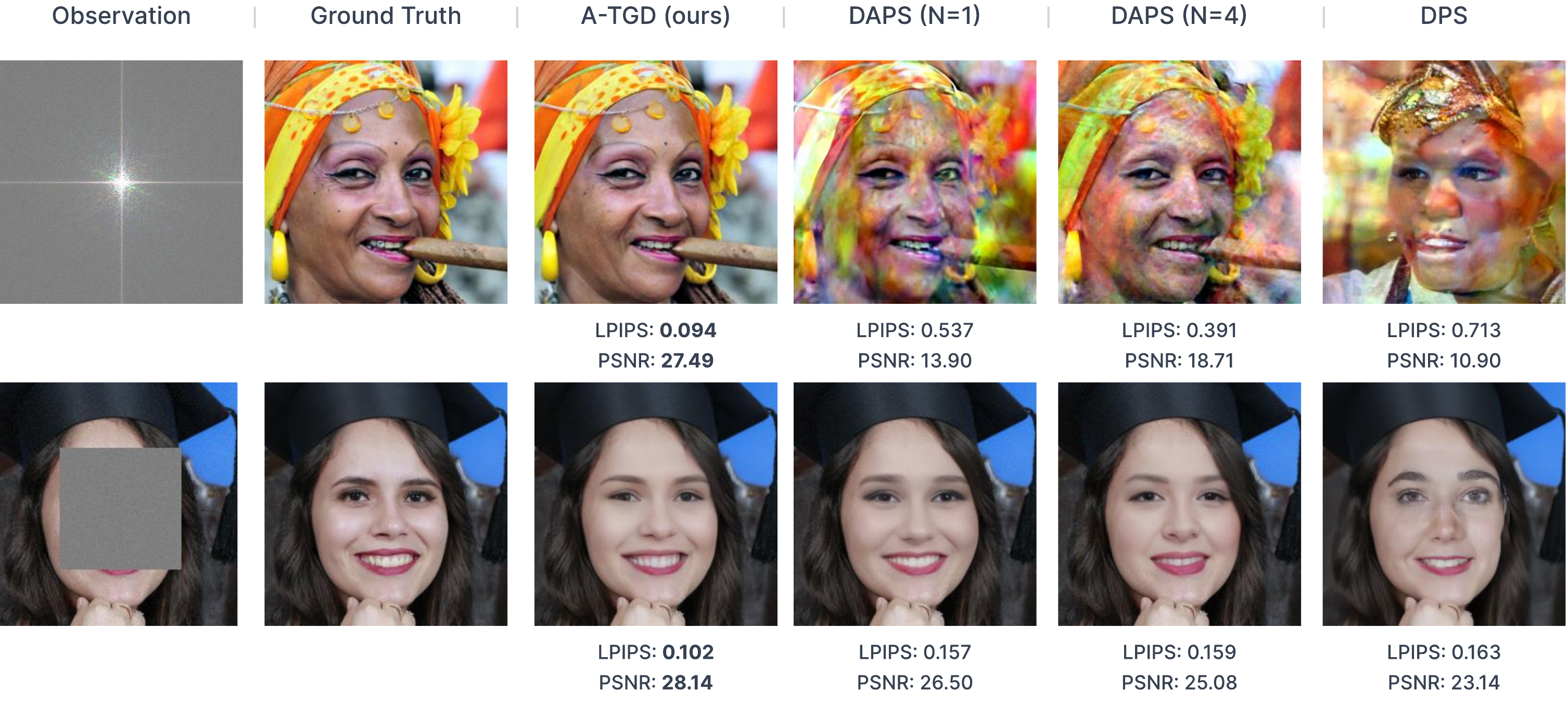}
    \caption{\textbf{Qualitative reconstructions on FFHQ.}
    We show one inpainting example and one phase retrieval example. A-TGD starts with four particles and prunes to one trajectory, while DAPS $(N=4)$ runs four independent trajectories and selects the final reconstruction with lowest measurement error.}
    \label{fig:qualitative}
\end{figure}

Figure~\ref{fig:qualitative} shows representative FFHQ reconstructions for inpainting and phase retrieval. The examples are consistent with Table~\ref{tab:image-results}: A-TGD preserves perceptual quality while improving observation consistency relative to independent trajectory baselines. In phase retrieval, DPS often fails to recover a plausible image, while DAPS improves substantially and A-TGD further benefits from early multi-trajectory exploration followed by pruning.

\subsection{Oracle trajectory selection}
\label{sec:oracle-selection}

Table~\ref{tab:oracle-selection} compares measurement-based selection
with oracle selection among four independent DAPS trajectories.
On ImageNet phase retrieval, oracle selection substantially improves
DAPS, showing that measurement error can fail to identify the best
available reconstruction. This effect is task-dependent: on FFHQ
phase retrieval, A-TGD outperforms even oracle-selected DAPS, while
on FFHQ inpainting it retains better LPIPS and SSIM but lower PSNR.
These comparisons diagnose selection among DAPS trajectories;
they do not directly quantify the effect of pruning in A-TGD.
Since ground-truth images are unavailable at test time, developing
better observation-based criteria for both particle pruning and
terminal trajectory selection remains a direction for future work.

\begin{table}[t]
\centering
\small
\setlength{\tabcolsep}{6pt}
\caption{Trajectory-selection comparisons (means over 100 validation
images per setting). DAPS-4 selects by lowest measurement error or
the ground-truth metric indicated in parentheses; all three metrics
are reported for the selected reconstruction.}
\label{tab:oracle-selection}
\begin{tabular}{lccc}
\toprule
Method / selection & PSNR $\uparrow$ & LPIPS $\downarrow$
& SSIM $\uparrow$ \\
\midrule
\multicolumn{4}{l}{\textbf{FFHQ inpainting}} \\
A-TGD                  & 25.14 & 0.122 & 0.853 \\
DAPS-4 (measurement)    & 24.52 & 0.157 & 0.790 \\
DAPS-4 (oracle LPIPS)   & 25.36 & 0.150 & 0.799 \\
DAPS-4 (oracle PSNR)    & 25.65 & 0.152 & 0.799 \\
DAPS-4 (oracle SSIM)    & 25.47 & 0.151 & 0.801 \\
\midrule
\multicolumn{4}{l}{\textbf{FFHQ phase retrieval}} \\
A-TGD                  & 29.43 & 0.172 & 0.811 \\
DAPS-4 (measurement)    & 27.81 & 0.207 & 0.765 \\
DAPS-4 (oracle LPIPS)   & 27.88 & 0.202 & 0.767 \\
DAPS-4 (oracle PSNR)    & 28.00 & 0.204 & 0.767 \\
DAPS-4 (oracle SSIM)    & 27.94 & 0.203 & 0.768 \\
\midrule
\multicolumn{4}{l}{\textbf{ImageNet phase retrieval}} \\
A-TGD                  & 22.18 & 0.326 & 0.617 \\
DAPS-4 (measurement)    & 22.41 & 0.325 & 0.593 \\
DAPS-4 (oracle LPIPS)   & 25.82 & 0.249 & 0.687 \\
DAPS-4 (oracle PSNR)    & 25.95 & 0.250 & 0.689 \\
DAPS-4 (oracle SSIM)    & 25.94 & 0.250 & 0.689 \\
\bottomrule
\end{tabular}
\end{table}

\section{Discussion and Limitations}
\label{sec:discussion}

TGD separates clean-space particle inference from the conditional reconstruction module, yielding an idealized SMC construction while allowing practical implementations to reuse existing training-free diffusion solvers. The theoretical guarantee assumes exact stagewise reconstruction and an unpruned particle procedure; image experiments instead use approximate reconstruction modules, learned unconditional scores, and A-TGD pruning to obtain a single reconstruction under a fixed compute budget. The usefulness of particle-based exploration depends on the inverse problem: it is most effective when early trajectories reach different plausible reconstruction basins and the likelihood reliably selects among them. When the image prior and observation effectively determine a single solution, aggressive early resampling or pruning can discard a good trajectory and increase variance, especially in challenging phase retrieval settings.

\section*{Acknowledgments and Disclosure of Funding}

The authors are grateful for the support of the Engineering and Physical Sciences Research Council (EPSRC) through the ProbAI Hub (grant EP/Y028783/1). This work was supported by Research England under the Expanding Excellence in England (E3) funding stream, which was awarded to MARS: Mathematics for AI in Real-world Systems in the School of Mathematical Sciences at Lancaster University. Authors declare no competing interests.

\bibliographystyle{plainnat}
\bibliography{references}

@Inbook{delmoral2004feynman,
author="Del Moral, Pierre",
title="Feynman-Kac Formulae",
bookTitle="Feynman-Kac Formulae: Genealogical and Interacting Particle Systems with Applications",
year="2004",
publisher="Springer New York",
address="New York, NY",
pages="47--93",
isbn="978-1-4684-9393-1",
doi="10.1007/978-1-4684-9393-1_2",
url="https://doi.org/10.1007/978-1-4684-9393-1_2"
}

@article{nemeth2019pseudo,
  title={Pseudo-extended markov chain monte carlo},
  author={Nemeth, Christopher and Lindsten, Fredrik and Filippone, Maurizio and Hensman, James},
  journal={Advances in Neural Information Processing Systems},
  volume={32},
  year={2019}
}

@inproceedings{
karras2022elucidatingdesignspacediffusionbased,
title={Elucidating the Design Space of Diffusion-Based Generative Models},
author={Tero Karras and Miika Aittala and Timo Aila and Samuli Laine},
booktitle={Advances in Neural Information Processing Systems},
editor={Alice H. Oh and Alekh Agarwal and Danielle Belgrave and Kyunghyun Cho},
year={2022},
url={https://openreview.net/forum?id=k7FuTOWMOc7}
}

@inproceedings{ho2020denoisingdiffusionprobabilisticmodels,
 author = {Ho, Jonathan and Jain, Ajay and Abbeel, Pieter},
 booktitle = {Advances in Neural Information Processing Systems},
 editor = {H. Larochelle and M. Ranzato and R. Hadsell and M.F. Balcan and H. Lin},
 pages = {6840--6851},
 publisher = {Curran Associates, Inc.},
 title = {Denoising Diffusion Probabilistic Models},
 url = {https://proceedings.neurips.cc/paper_files/paper/2020/file/4c5bcfec8584af0d967f1ab10179ca4b-Paper.pdf},
 volume = {33},
 year = {2020}
}

@inproceedings{
song2021scorebasedgenerativemodelingstochastic,
title={Score-Based Generative Modeling through Stochastic Differential Equations},
author={Yang Song and Jascha Sohl-Dickstein and Diederik P Kingma and Abhishek Kumar and Stefano Ermon and Ben Poole},
booktitle={International Conference on Learning Representations},
year={2021},
url={https://openreview.net/forum?id=PxTIG12RRHS}
}

@misc{ho2022classifierfreediffusionguidance,
      title={Classifier-Free Diffusion Guidance}, 
      author={Jonathan Ho and Tim Salimans},
      year={2022},
      eprint={2207.12598},
      archivePrefix={arXiv},
      primaryClass={cs.LG},
      url={https://arxiv.org/abs/2207.12598}, 
}

@inproceedings{
dhariwal2021diffusionmodelsbeatgans,
title={Diffusion Models Beat {GAN}s on Image Synthesis},
author={Prafulla Dhariwal and Alexander Quinn Nichol},
booktitle={Advances in Neural Information Processing Systems},
editor={A. Beygelzimer and Y. Dauphin and P. Liang and J. Wortman Vaughan},
year={2021},
url={https://openreview.net/forum?id=AAWuCvzaVt}
}

@inproceedings{
chung2024diffusionposteriorsamplinggeneral,
title={Diffusion Posterior Sampling for General Noisy Inverse Problems},
author={Hyungjin Chung and Jeongsol Kim and Michael Thompson Mccann and Marc Louis Klasky and Jong Chul Ye},
booktitle={The Eleventh International Conference on Learning Representations },
year={2023},
url={https://openreview.net/forum?id=OnD9zGAGT0k}
}

@inproceedings{sharrock2024tuning,
  title={Tuning-free maximum likelihood training of latent variable models via coin betting},
  author={Sharrock, Louis and Dodd, Daniel and Nemeth, Christopher},
  booktitle={International Conference on Artificial Intelligence and Statistics},
  pages={1810--1818},
  year={2024},
  organization={PMLR}
}

@inproceedings{
he2023manifoldpreservingguideddiffusion,
title={Manifold Preserving Guided Diffusion},
author={Yutong He and Naoki Murata and Chieh-Hsin Lai and Yuhta Takida and Toshimitsu Uesaka and Dongjun Kim and Wei-Hsiang Liao and Yuki Mitsufuji and J Zico Kolter and Ruslan Salakhutdinov and Stefano Ermon},
booktitle={The Twelfth International Conference on Learning Representations},
year={2024},
url={https://openreview.net/forum?id=o3BxOLoxm1}
}

@inproceedings{Zhang_2025,
   title={Improving Diffusion Inverse Problem Solving with Decoupled Noise Annealing},
   url={http://dx.doi.org/10.1109/CVPR52734.2025.01946},
   DOI={10.1109/cvpr52734.2025.01946},
   booktitle={2025 IEEE/CVF Conference on Computer Vision and Pattern Recognition (CVPR)},
   publisher={IEEE},
   author={Zhang, Bingliang and Chu, Wenda and Berner, Julius and Meng, Chenlin and Anandkumar, Anima and Song, Yang},
   year={2025},
   month=jun, pages = {20895--20905}, }

@article{vyner2023swiss,
  title={SwISS: A scalable Markov chain Monte Carlo divide-and-conquer strategy},
  author={Vyner, Callum and Nemeth, Christopher and Sherlock, Chris},
  journal={Stat},
  volume={12},
  number={1},
  pages={e523},
  year={2023},
  publisher={Wiley Online Library}
}

@misc{ahmed2025solvingdiffusioninverseproblems,
      title={Solving Diffusion Inverse Problems with Restart Posterior Sampling}, 
      author={Bilal Ahmed and Joseph G. Makin},
      year={2025},
      eprint={2511.20705},
      archivePrefix={arXiv},
      primaryClass={cs.LG},
      url={https://arxiv.org/abs/2511.20705}, 
}

@inproceedings{
wu2024practicalasymptoticallyexactconditional,
title={Practical and Asymptotically Exact Conditional Sampling in Diffusion Models},
author={Brian L. Trippe and Luhuan Wu and Christian A. Naesseth and David Blei and John Patrick Cunningham},
booktitle={ICML 2023 Workshop on Structured Probabilistic Inference {\&} Generative Modeling},
year={2023},
url={https://openreview.net/forum?id=r9s3Gbxz7g}
}

@article{cabezas2024markovian,
  title={Markovian flow matching: Accelerating MCMC with continuous normalizing flows},
  author={Cabezas, Alberto and Sharrock, Louis and Nemeth, Christopher},
  journal={Advances in Neural Information Processing Systems},
  volume={37},
  pages={104383--104411},
  year={2024}
}

@book{fearnhead2025scalable, place={Cambridge}, series={Institute of Mathematical Statistics Monographs}, title={Scalable Monte Carlo for Bayesian Learning}, publisher={Cambridge University Press}, author={Fearnhead, Paul and Nemeth, Christopher and Oates, Chris J. and Sherlock, Chris}, year={2025}, collection={Institute of Mathematical Statistics Monographs}}

@inproceedings{tfg,
 author = {Ye, Haotian and Lin, Haowei and Han, Jiaqi and Xu, Minkai and Liu, Sheng and Liang, Yitao and Ma, Jianzhu and Zou, James and Ermon, Stefano},
 booktitle = {Advances in Neural Information Processing Systems},
 doi = {10.52202/079017-0704},
 editor = {A. Globerson and L. Mackey and D. Belgrave and A. Fan and U. Paquet and J. Tomczak and C. Zhang},
 pages = {22370--22417},
 publisher = {Curran Associates, Inc.},
 title = {TFG: Unified Training-Free Guidance for Diffusion Models},
 url = {https://proceedings.neurips.cc/paper_files/paper/2024/file/2818054fc6de6dacdda0f142a3475933-Paper-Conference.pdf},
 volume = {37},
 year = {2024}
}

@inproceedings{
trippe2023diffusion,
title={Diffusion Probabilistic Modeling of Protein Backbones in 3D for the motif-scaffolding problem},
author={Brian L. Trippe and Jason Yim and Doug Tischer and David Baker and Tamara Broderick and Regina Barzilay and Tommi S. Jaakkola},
booktitle={The Eleventh International Conference on Learning Representations },
year={2023},
url={https://openreview.net/forum?id=6TxBxqNME1Y}
}

@inproceedings{
fps_smc,
title={Diffusion Posterior Sampling for Linear Inverse Problem Solving: A Filtering Perspective},
author={Zehao Dou and Yang Song},
booktitle={The Twelfth International Conference on Learning Representations},
year={2024},
url={https://openreview.net/forum?id=tplXNcHZs1}
}

@inproceedings{
autogduidance,
title={Guiding a Diffusion Model with a Bad Version of Itself},
author={Tero Karras and Miika Aittala and Tuomas Kynk{\"a}{\"a}nniemi and Jaakko Lehtinen and Timo Aila and Samuli Laine},
booktitle={The Thirty-eighth Annual Conference on Neural Information Processing Systems},
year={2024},
url={https://openreview.net/forum?id=bg6fVPVs3s}
}

@inproceedings{
free_hunch,
title={Free Hunch: Denoiser Covariance Estimation for Diffusion Models Without Extra Costs},
author={Severi Rissanen and Markus Heinonen and Arno Solin},
booktitle={The Thirteenth International Conference on Learning Representations},
year={2025},
url={https://openreview.net/forum?id=4JK2XMGUc8}
}

@inproceedings{
bansal2024universal,
title={Universal Guidance for Diffusion Models},
author={Arpit Bansal and Hong-Min Chu and Avi Schwarzschild and Roni Sengupta and Micah Goldblum and Jonas Geiping and Tom Goldstein},
booktitle={The Twelfth International Conference on Learning Representations},
year={2024},
url={https://openreview.net/forum?id=pzpWBbnwiJ}
}

@inproceedings{
pgdm,
title={Pseudoinverse-Guided Diffusion Models for Inverse Problems},
author={Jiaming Song and Arash Vahdat and Morteza Mardani and Jan Kautz},
booktitle={International Conference on Learning Representations},
year={2023},
url={https://openreview.net/forum?id=9_gsMA8MRKQ}
}

@article{ANDERSON1982313,
title = {Reverse-time diffusion equation models},
journal = {Stochastic Processes and their Applications},
volume = {12},
number = {3},
pages = {313-326},
year = {1982},
issn = {0304-4149},
doi = {10.1016/0304-4149(82)90051-5},
url = {https://www.sciencedirect.com/science/article/pii/0304414982900515},
author = {Brian D.O. Anderson}
}

@book{murphy2022probabilistic,
  title={Probabilistic machine learning: an introduction},
  author={Murphy, Kevin P},
  year={2022},
  publisher={MIT press}
}

@article{yu2023freedom,
title={FreeDoM: Training-Free Energy-Guided Conditional Diffusion Model},
author={Yu, Jiwen and Wang, Yinhuai and Zhao, Chen and Ghanem, Bernard and Zhang, Jian},
journal={Proceedings of the IEEE/CVF International Conference on Computer Vision (ICCV)},
year={2023}
}

@InProceedings{repaint,
  author    = {Lugmayr, Andreas and Danelljan, Martin and Romero, Andres and Yu, Fisher and Timofte, Radu and Van Gool, Luc},
  title     = {RePaint: Inpainting Using Denoising Diffusion Probabilistic Models},
  booktitle = {Proceedings of the IEEE/CVF Conference on Computer Vision and Pattern Recognition (CVPR)},
  year      = {2022},
}

@book{gelman1995bayesian,
  title={Bayesian data analysis},
  author={Gelman, Andrew and Carlin, John B and Stern, Hal S and Rubin, Donald B},
  year={1995},
  publisher={Chapman and Hall/CRC}
}

@article{smc_elements,
author = {Naesseth, Christian A. and Lindsten, Fredrik and Sch\"{o}n, Thomas B.},
title = {Elements of Sequential Monte Carlo},
year = {2019},
issue_date = {Nov 2019},
publisher = {Now Publishers Inc.},
address = {Hanover, MA, USA},
volume = {12},
number = {3},
issn = {1935-8237},
url = {https://doi.org/10.1561/2200000074},
doi = {10.1561/2200000074},
journal = {Found. Trends Mach. Learn.},
month = nov,
pages = {307–392},
numpages = {123}
}

@article{bonneel2015sliced,
  author  = {Bonneel, Nicolas and Rabin, Julien and Peyr{\'e}, Gabriel and Pfister, Hanspeter},
  title   = {Sliced and Radon Wasserstein Barycenters of Measures},
  journal = {Journal of Mathematical Imaging and Vision},
  volume  = {51},
  pages   = {22--45},
  year    = {2015},
  doi     = {10.1007/s10851-014-0506-3},
  url     = {https://doi.org/10.1007/s10851-014-0506-3}
}

@misc{flamary2024pot,
  author = {Flamary, R{\'e}mi and Vincent-Cuaz, C{\'e}dric and Courty, Nicolas and Gramfort, Alexandre and Kachaiev, Oleksii and Quang Tran, Huy and David, Laurène and Bonet, Cl{\'e}ment and Cassereau, Nathan and Gnassounou, Th{\'e}o and Tanguy, Eloi and Delon, Julie and Collas, Antoine and Mazelet, Sonia and Chapel, Laetitia and Kerdoncuff, Tanguy and Yu, Xizheng and Feickert, Matthew and Krzakala, Paul and Liu, Tianlin and Fernandes Montesuma, Eduardo},
  title = {POT Python Optimal Transport (version 0.9.5)},
  url = {https://github.com/PythonOT/POT},
  year = {2024}
}

@article{flamary2021pot,
  author  = {R{\'e}mi Flamary and Nicolas Courty and Alexandre Gramfort and Mokhtar Z. Alaya and Aur{\'e}lie Boisbunon and Stanislas Chambon and Laetitia Chapel and Adrien Corenflos and Kilian Fatras and Nemo Fournier and L{\'e}o Gautheron and Nathalie T.H. Gayraud and Hicham Janati and Alain Rakotomamonjy and Ievgen Redko and Antoine Rolet and Antony Schutz and Vivien Seguy and Danica J. Sutherland and Romain Tavenard and Alexander Tong and Titouan Vayer},
  title   = {POT: Python Optimal Transport},
  journal = {Journal of Machine Learning Research},
  year    = {2021},
  volume  = {22},
  number  = {78},
  pages   = {1-8},
  url     = {http://jmlr.org/papers/v22/20-451.html}
}

@article{neal1996sampling,
  title={Sampling from multimodal distributions using tempered transitions},
  author={Neal, Radford M},
  journal={Statistics and computing},
  volume={6},
  number={4},
  pages={353--366},
  year={1996},
  publisher={Springer}
}

@InProceedings{sohldickstein2015deep,
  title = 	 {Deep Unsupervised Learning using Nonequilibrium Thermodynamics},
  author = 	 {Sohl-Dickstein, Jascha and Weiss, Eric and Maheswaranathan, Niru and Ganguli, Surya},
  booktitle = 	 {Proceedings of the 32nd International Conference on Machine Learning},
  pages = 	 {2256--2265},
  year = 	 {2015},
  editor = 	 {Bach, Francis and Blei, David},
  volume = 	 {37},
  series = 	 {Proceedings of Machine Learning Research},
  address = 	 {Lille, France},
  month = 	 {07--09 Jul},
  publisher =    {PMLR},
  url = 	 {https://proceedings.mlr.press/v37/sohl-dickstein15.html}
}

@inproceedings{song2019generative,
 author = {Song, Yang and Ermon, Stefano},
 booktitle = {Advances in Neural Information Processing Systems},
 editor = {H. Wallach and H. Larochelle and A. Beygelzimer and F. d\textquotesingle Alch\'{e}-Buc and E. Fox and R. Garnett},
 pages = {},
 publisher = {Curran Associates, Inc.},
 title = {Generative Modeling by Estimating Gradients of the Data Distribution},
 url = {https://proceedings.neurips.cc/paper_files/paper/2019/file/3001ef257407d5a371a96dcd947c7d93-Paper.pdf},
 volume = {32},
 year = {2019}
}

@article{scott2022bayes,
author = {Steven L. Scott and Alexander W. Blocker and Fernando V. Bonassi and Hugh A. Chipman and Edward I. George and Robert E. McCulloch},
title = {Bayes and big data: the consensus Monte Carlo algorithm},
journal = {International Journal of Management Science and Engineering Management},
volume = {11},
number = {2},
pages = {78--88},
year = {2016},
publisher = {Taylor \& Francis},
doi = {10.1080/17509653.2016.1142191},


URL = { 
    
        https://doi.org/10.1080/17509653.2016.1142191
    
    

},
eprint = { 
    
        https://doi.org/10.1080/17509653.2016.1142191
    
    

}

}

@inproceedings{kuntz2023particle,
  title={Particle algorithms for maximum likelihood training of latent variable models},
  author={Kuntz, Juan and Lim, Jen Ning and Johansen, Adam M},
  booktitle={International Conference on Artificial Intelligence and Statistics},
  pages={5134--5180},
  year={2023},
  organization={PMLR}
}

@inproceedings{song2020improved,
 author = {Song, Yang and Ermon, Stefano},
 booktitle = {Advances in Neural Information Processing Systems},
 editor = {H. Larochelle and M. Ranzato and R. Hadsell and M.F. Balcan and H. Lin},
 pages = {12438--12448},
 publisher = {Curran Associates, Inc.},
 title = {Improved Techniques for Training Score-Based Generative Models},
 url = {https://proceedings.neurips.cc/paper_files/paper/2020/file/92c3b916311a5517d9290576e3ea37ad-Paper.pdf},
 volume = {33},
 year = {2020}
}

@inproceedings{
song2021denoising,
title={Denoising Diffusion Implicit Models},
author={Jiaming Song and Chenlin Meng and Stefano Ermon},
booktitle={International Conference on Learning Representations},
year={2021},
url={https://openreview.net/forum?id=St1giarCHLP}
}

@inproceedings{nichol2021improved,
  title     = {Improved Denoising Diffusion Probabilistic Models},
  author    = {Nichol, Alex and Dhariwal, Prafulla},
  booktitle = {Proceedings of the 38th International Conference on Machine Learning},
  series    = {Proceedings of Machine Learning Research},
  volume    = {139},
  pages     = {8162--8171},
  year      = {2021},
  publisher = {PMLR},
  url       = {https://proceedings.mlr.press/v139/nichol21a.html}
}

@inproceedings{
pandey2025variational,
title={Variational Control for Guidance in Diffusion Models},
author={Kushagra Pandey and Farrin Marouf Sofian and Felix Draxler and Theofanis Karaletsos and Stephan Mandt},
booktitle={Forty-second International Conference on Machine Learning},
year={2025},
url={https://openreview.net/forum?id=Z0ffRRtOim}
}

@inproceedings{
gao2025rectified,
title={{REG}: Rectified Gradient Guidance for Conditional Diffusion Models},
author={Zhengqi Gao and Kaiwen Zha and Tianyuan Zhang and Zihui Xue and Duane S Boning},
booktitle={Forty-second International Conference on Machine Learning},
year={2025},
url={https://openreview.net/forum?id=uNK4ftGdnq}
}

@inproceedings{
song2022solving,
title={Solving Inverse Problems in Medical Imaging with Score-Based Generative Models},
author={Yang Song and Liyue Shen and Lei Xing and Stefano Ermon},
booktitle={International Conference on Learning Representations},
year={2022},
url={https://openreview.net/forum?id=vaRCHVj0uGI}
}

@article{kawar2021snips,
  title={{SNIPS}: Solving noisy inverse problems stochastically},
  author={Kawar, Bahjat and Vaksman, Gregory and Elad, Michael},
  journal={Advances in Neural Information Processing Systems},
  volume={34},
  pages={21757--21769},
  year={2021}
}

@inproceedings{kawar2022denoising,
 author = {Kawar, Bahjat and Elad, Michael and Ermon, Stefano and Song, Jiaming},
 booktitle = {Advances in Neural Information Processing Systems},
 editor = {S. Koyejo and S. Mohamed and A. Agarwal and D. Belgrave and K. Cho and A. Oh},
 pages = {23593--23606},
 publisher = {Curran Associates, Inc.},
 title = {Denoising Diffusion Restoration Models},
 url = {https://proceedings.neurips.cc/paper_files/paper/2022/file/95504595b6169131b6ed6cd72eb05616-Paper-Conference.pdf},
 volume = {35},
 year = {2022}
}

@inproceedings{
wang2023zero,
title={Zero-Shot Image Restoration Using Denoising Diffusion Null-Space Model},
author={Yinhuai Wang and Jiwen Yu and Jian Zhang},
booktitle={The Eleventh International Conference on Learning Representations },
year={2023},
url={https://openreview.net/forum?id=mRieQgMtNTQ}
}

@inproceedings{zhu2023denoising,
      title={Denoising Diffusion Models for Plug-and-Play Image Restoration},
      author={Yuanzhi Zhu and Kai Zhang and Jingyun Liang and Jiezhang Cao and Bihan Wen and Radu Timofte and Luc Van Gool},
      booktitle={IEEE Conference on Computer Vision and Pattern Recognition Workshops (NTIRE)},
      year={2023},
}

@inproceedings{
mardani2023variational,
title={A Variational Perspective on Solving Inverse Problems with Diffusion Models},
author={Morteza Mardani and Jiaming Song and Jan Kautz and Arash Vahdat},
booktitle={The Twelfth International Conference on Learning Representations},
year={2024},
url={https://openreview.net/forum?id=1YO4EE3SPB}
}

@inproceedings{murata2023gibbsddrm,
  title={Gibbs{DDRM}: A Partially Collapsed Gibbs Sampler for Solving Blind Inverse Problems with Denoising Diffusion Restoration},
  author={Murata, Naoki and Saito, Koichi and Lai, Chieh-Hsin and Takida, Yuhta and Uesaka, Toshimitsu and Mitsufuji, Yuki and Ermon, Stefano},
  booktitle={International Conference on Machine Learning},
  year={2023}
}

@inproceedings{
chung2022improving,
title={Improving Diffusion Models for Inverse Problems using Manifold Constraints},
author={Hyungjin Chung and Byeongsu Sim and Dohoon Ryu and Jong Chul Ye},
booktitle={Advances in Neural Information Processing Systems},
editor={Alice H. Oh and Alekh Agarwal and Danielle Belgrave and Kyunghyun Cho},
year={2022},
url={https://openreview.net/forum?id=nJJjv0JDJju}
}

@book{doucet2001sequential,
  title     = {Sequential Monte Carlo Methods in Practice},
  author    = {Doucet, Arnaud and de Freitas, Nando and Gordon, Neil},
  year      = {2001},
  publisher = {Springer},
  address   = {New York},
  series    = {Statistics for Engineering and Information Science}
}

@book{chopin2004introduction,
  title     = {An Introduction to Sequential {Monte Carlo}},
  author    = {Chopin, Nicolas and Papaspiliopoulos, Omiros},
  publisher = {Springer},
  address   = {Cham},
  series    = {Springer Series in Statistics},
  year      = {2020},
  doi       = {10.1007/978-3-030-47845-2}
}

@article{delmoral2006sequential,
  title   = {Sequential Monte Carlo Samplers},
  author  = {Del Moral, Pierre and Doucet, Arnaud and Jasra, Ajay},
  journal = {Journal of the Royal Statistical Society: Series B (Statistical Methodology)},
  volume  = {68},
  number  = {3},
  pages   = {411--436},
  year    = {2006},
  publisher = {Wiley}
}

@inproceedings{karras2019style,
  title     = {A Style-Based Generator Architecture for Generative Adversarial Networks},
  author    = {Karras, Tero and Laine, Samuli and Aila, Timo},
  booktitle = {Proceedings of the IEEE/CVF Conference on Computer Vision and Pattern Recognition (CVPR)},
  pages     = {4401--4410},
  year      = {2019}
}

@inproceedings{deng2009imagenet,
  title     = {ImageNet: A Large-Scale Hierarchical Image Database},
  author    = {Deng, Jia and Dong, Wei and Socher, Richard and Li, Li-Jia and Li, Kai and Fei-Fei, Li},
  booktitle = {Proceedings of the IEEE Conference on Computer Vision and Pattern Recognition (CVPR)},
  pages     = {248--255},
  year      = {2009}
}

@inproceedings{zhang2018unreasonable,
  title     = {The Unreasonable Effectiveness of Deep Features as a Perceptual Metric},
  author    = {Zhang, Richard and Isola, Phillip and Efros, Alexei A. and Shechtman, Eli and Wang, Oliver},
  booktitle = {Proceedings of the IEEE Conference on Computer Vision and Pattern Recognition (CVPR)},
  pages     = {586--595},
  year      = {2018}
}

@article{neal1998annealedimportancesampling,
  author  = {Neal, Radford M.},
  title   = {Annealed importance sampling},
  journal = {Statistics and Computing},
  year    = {2001},
  month   = apr,
  volume  = {11},
  number  = {2},
  pages   = {125--139},
  issn    = {1573-1375},
  doi     = {10.1023/A:1008923215028},
  url     = {https://doi.org/10.1023/A:1008923215028}
}

\newpage

\appendix

\section{Lifted SMC Construction and Incremental Weights}
\label{app:smc-weight-derivations}

We derive the incremental weight update used in the stagewise TGD construction.
Throughout this appendix, densities are understood up to normalizing constants that do not depend on the particle index, since SMC weights are normalized after each outer stage. For compactness, write
\[
    \ell(\bx_0) := p(\by \mid \bx_0).
\]

Recall that the clean-space annealed target at outer stage \(r\) is
\begin{equation}
    \pi_r(\bx_0)
    \propto
    p(\bx_0)\ell(\bx_0)^{\lambda_r},
    \qquad
    0 \leq \lambda_R \leq \cdots \leq \lambda_0 = 1 .
    \label{eq:app-clean-target}
\end{equation}
To connect this target to the diffusion prior, TGD introduces an auxiliary noisy state
\(\bz_r\) at noise level \(s_r\). Define the unnormalized lifted density
\begin{equation}
    \gamma_r(\bx_0,\bz_r)
    :=
    p(\bx_0)\ell(\bx_0)^{\lambda_r}p_{s_r}(\bz_r \mid \bx_0).
    \label{eq:app-lifted-density}
\end{equation}
Its clean marginal is proportional to \(\pi_r\), since
\(\int p_{s_r}(\bz_r\mid \bx_0)\,d\bz_r = 1\). Its noisy marginal is
\begin{equation}
    \eta_r(\bz_r)
    :=
    \int \gamma_r(\bx_0,\bz_r)\,d\bx_0 .
    \label{eq:app-noisy-marginal}
\end{equation}

The ideal TGD reconstruction kernel at stage \(r\) is the conditional distribution induced by
\(\gamma_r\):
\begin{equation}
    K_r(d\bx_0 \mid \bz_r)
    :=
    p_{\lambda_r,s_r}(\bx_0 \mid \bz_r,\by)\,d\bx_0
    =
    \frac{\gamma_r(\bx_0,\bz_r)}{\eta_r(\bz_r)}\,d\bx_0 .
    \label{eq:app-reconstruction-kernel}
\end{equation}
Equivalently,
\begin{equation}
    p_{\lambda_r,s_r}(\bx_0 \mid \bz_r,\by)
    \propto
    p_{s_r}(\bx_0\mid \bz_r)\ell(\bx_0)^{\lambda_r}
    \propto
    p(\bx_0)p_{s_r}(\bz_r\mid \bx_0)\ell(\bx_0)^{\lambda_r}.
    \label{eq:app-reconstruction-law}
\end{equation}

Suppose a particle \(\bz_r\) targets the noisy marginal \(\eta_r\). One ideal TGD transition first samples
\[
    \bx_0 \sim K_r(\cdot \mid \bz_r),
\]
and then re-noises the retained clean sample to the next noise level:
\[
    \bz_{r-1} \sim p_{s_{r-1}}(\cdot \mid \bx_0).
\]
Thus the proposal kernel from \(\bz_r\) to \((\bx_0,\bz_{r-1})\) is
\begin{equation}
    Q_r(d\bx_0,d\bz_{r-1}\mid \bz_r)
    =
    K_r(d\bx_0\mid \bz_r)\,
    p_{s_{r-1}}(d\bz_{r-1}\mid \bx_0).
    \label{eq:app-proposal-kernel}
\end{equation}

To derive the incremental weight, define the following unnormalized extended target:
\begin{equation}
    \Gamma_{r-1}(\bz_r,\bx_0,\bz_{r-1})
    :=
    p(\bx_0)\ell(\bx_0)^{\lambda_{r-1}}
    p_{s_r}(\bz_r\mid \bx_0)
    p_{s_{r-1}}(\bz_{r-1}\mid \bx_0).
    \label{eq:app-extended-target}
\end{equation}
This is a valid extension of the next-stage lifted target because marginalizing out \(\bz_r\) gives
\begin{align}
    \int \Gamma_{r-1}(\bz_r,\bx_0,\bz_{r-1})\,d\bz_r
    &=
    p(\bx_0)\ell(\bx_0)^{\lambda_{r-1}}
    p_{s_{r-1}}(\bz_{r-1}\mid \bx_0)
    \int p_{s_r}(\bz_r\mid \bx_0)\,d\bz_r \nonumber \\
    &=
    p(\bx_0)\ell(\bx_0)^{\lambda_{r-1}}
    p_{s_{r-1}}(\bz_{r-1}\mid \bx_0)
    =
    \gamma_{r-1}(\bx_0,\bz_{r-1}).
    \label{eq:app-extended-marginal}
\end{align}

The standard SMC incremental weight associated with the proposal
\(\eta_r(\bz_r)Q_r(d\bx_0,d\bz_{r-1}\mid \bz_r)\) and the extended target
\(\Gamma_{r-1}\) is
\begin{equation}
    G_r(\bz_r,\bx_0,\bz_{r-1})
    =
    \frac{
        \Gamma_{r-1}(\bz_r,\bx_0,\bz_{r-1})
    }{
        \eta_r(\bz_r)Q_r(\bx_0,\bz_{r-1}\mid \bz_r)
    },
    \label{eq:app-incremental-weight-general}
\end{equation}
up to a stage-dependent normalizing constant.

Using \eqref{eq:app-reconstruction-kernel} and \eqref{eq:app-proposal-kernel}, the denominator becomes
\begin{align}
    \eta_r(\bz_r)Q_r(\bx_0,\bz_{r-1}\mid \bz_r)
    &=
    \eta_r(\bz_r)
    \frac{\gamma_r(\bx_0,\bz_r)}{\eta_r(\bz_r)}
    p_{s_{r-1}}(\bz_{r-1}\mid \bx_0) \nonumber \\
    &=
    \gamma_r(\bx_0,\bz_r)
    p_{s_{r-1}}(\bz_{r-1}\mid \bx_0) \nonumber \\
    &=
    p(\bx_0)\ell(\bx_0)^{\lambda_r}
    p_{s_r}(\bz_r\mid \bx_0)
    p_{s_{r-1}}(\bz_{r-1}\mid \bx_0).
    \label{eq:app-denominator}
\end{align}
Comparing \eqref{eq:app-denominator} with the extended target
\eqref{eq:app-extended-target}, all terms cancel except the change in likelihood tempering:
\begin{equation}
    G_r(\bz_r,\bx_0,\bz_{r-1})
    \propto
    \ell(\bx_0)^{\lambda_{r-1}-\lambda_r}
    =
    p(\by\mid \bx_0)^{\lambda_{r-1}-\lambda_r}.
    \label{eq:app-incremental-potential}
\end{equation}
Therefore, for particle \(i\),
\begin{equation}
    \widetilde w_{r-1}^{(i)}
    =
    w_r^{(i)}
    p(\by\mid \bx_0^{(i)})^{\lambda_{r-1}-\lambda_r}.
    \label{eq:app-tgd-weight-update}
\end{equation}

The incremental weight does not depend on the proposed auxiliary state \(\bz_{r-1}\). Consequently, the weighting and resampling step can be performed immediately after reconstruction, before re-noising to the next level. This gives exactly the stagewise update used by ideal TGD in Section~\ref{sec:stagewise}. 

\section{Consistency of Ideal TGD}
\label{app:consistency-tgd}

We state and prove the asymptotic consistency result for ideal TGD: the version of TGD in which each stagewise reconstruction step samples exactly from the tempered conditional law introduced in Section~\ref{sec:stagewise}. This result applies to the full particle procedure with exact reconstruction kernels. It does not apply to approximate reconstruction modules or to the pruning step used by A-TGD.

Throughout this appendix, all densities are taken with respect to fixed dominating measures, and we use the same notation for a density and its associated measure when the meaning is clear. For compactness, write
\[
    \ell(\bx_0) := p(\by \mid \bx_0).
\]

For \(r=0,\dots,R\), define the unnormalized clean-space density
\begin{equation}
    \gamma_r(\bx_0)
    :=
    p(\bx_0)\ell(\bx_0)^{\lambda_r},
    \qquad
    Z_r := \int \gamma_r(\bu)\,d\bu .
    \label{eq:appB-clean-unnormalized}
\end{equation}
The corresponding normalized clean-space target is
\begin{equation}
    \pi_r(d\bx_0)
    :=
    \frac{\gamma_r(\bx_0)}{Z_r}\,d\bx_0 .
    \label{eq:appB-clean-target}
\end{equation}
Since \(\lambda_0=1\), \(\pi_0\) is the posterior \(p(\bx_0\mid \by)\).

At outer stage \(r\), TGD uses an auxiliary noisy state \(\bz_r\) at noise level \(s_r\). The corresponding noisy auxiliary marginal is
\begin{equation}
    \eta_r(\bz_r)
    :=
    \int
    \frac{\gamma_r(\bx_0)}{Z_r}
    p_{s_r}(\bz_r\mid \bx_0)\,d\bx_0,
    \qquad
    \eta_r(d\bz_r) := \eta_r(\bz_r)\,d\bz_r .
    \label{eq:appB-noisy-marginal}
\end{equation}

\begin{theorem}[Consistency of ideal TGD]
\label{thm:appB-consistency}
Assume the following conditions.

\begin{enumerate}
    \item \textbf{Well-defined targets.}
    For every \(r=0,\dots,R\), \(0<Z_r<\infty\), and the tempering schedule satisfies
    \[
        0 \leq \lambda_R \leq \lambda_{R-1} \leq \cdots \leq \lambda_0 = 1 .
    \]

    \item \textbf{Initialization.}
    \(\lambda_R=0\), and the initial particles are sampled i.i.d. from the noisy prior:
    \[
        \bz_R^{(i)} \sim p_{s_R}(\bz_R), \qquad w_R^{(i)}=1/N .
    \]
    Equivalently, the initial weighted empirical measure is consistent for \(\eta_R\).

    \item \textbf{Exact reconstruction.}
    For each \(r=0,\dots,R\), the stagewise reconstruction kernel used by ideal TGD is the exact conditional law
    \begin{equation}
        K_r(d\bx_0\mid \bz_r)
        :=
        p_{\lambda_r,s_r}(\bx_0\mid \bz_r,\by)\,d\bx_0
        =
        \frac{
            \pi_r(\bx_0)p_{s_r}(\bz_r\mid \bx_0)
        }{
            \eta_r(\bz_r)
        }\,d\bx_0 .
        \label{eq:appB-exact-reconstruction}
    \end{equation}
    Equivalently,
    \begin{equation}
        p_{\lambda_r,s_r}(\bx_0\mid \bz_r,\by)
        \propto
        p_{s_r}(\bx_0\mid \bz_r)\ell(\bx_0)^{\lambda_r}.
        \label{eq:appB-reconstruction-density}
    \end{equation}

    \item \textbf{Bounded incremental potentials.}
    For each \(r=R,\dots,1\), the incremental potential
    \begin{equation}
        G_r(\bx_0)
        :=
        \ell(\bx_0)^{\lambda_{r-1}-\lambda_r}
        =
        p(\by\mid \bx_0)^{\lambda_{r-1}-\lambda_r}
        \label{eq:appB-incremental-potential}
    \end{equation}
    is measurable and bounded, with \(\pi_r(G_r)>0\).

    \item \textbf{Consistent resampling.} At each outer stage, either resampling is omitted or the resampling scheme preserves consistency of weighted empirical measures: whenever the weighted particle system is consistent for a target distribution, the equally weighted resampled particles are also consistent for the same target. Standard SMC schemes \cite{chopin2004introduction, smc_elements} such as multinomial, stratified, residual, and systematic resampling satisfy this condition.
\end{enumerate}

Then, for every bounded measurable test function \(h\), the final weighted empirical measure produced by ideal TGD satisfies
\begin{equation}
    \sum_{i=1}^N w_0^{(i)}h(\bx_0^{(i)})
    \xrightarrow{\mathbb{P}}
    \int h(\bx_0)\,\pi_0(d\bx_0)
    =
    \mathbb{E}_{p(\bx_0\mid \by)}[h(\bx_0)] .
    \label{eq:appB-consistency-result}
\end{equation}
In particular, the final weighted empirical measure produced by ideal TGD converges weakly in probability to the posterior \(p(\bx_0\mid \by)\) as \(N\to\infty\).
\end{theorem}

\begin{proof}
We prove the result by identifying ideal TGD as a standard SMC approximation of an exact sequence of measure transformations.

\paragraph{Step 1: One ideal TGD stage has the correct population flow.}

Fix \(r\in\{R,\dots,1\}\). By the definition of the exact reconstruction kernel,
\begin{equation}
    K_r(d\bx_0\mid \bz_r)
    =
    \frac{
        \pi_r(\bx_0)p_{s_r}(\bz_r\mid \bx_0)
    }{
        \eta_r(\bz_r)
    }\,d\bx_0 .
    \label{eq:appB-kernel-bayes}
\end{equation}
Integrating this kernel against \(\eta_r\) gives
\begin{align}
    \eta_r K_r(d\bx_0)
    &:=
    \int \eta_r(d\bz_r)K_r(d\bx_0\mid \bz_r) \nonumber \\
    &=
    \int
    \eta_r(\bz_r)\,d\bz_r
    \frac{
        \pi_r(\bx_0)p_{s_r}(\bz_r\mid \bx_0)
    }{
        \eta_r(\bz_r)
    }\,d\bx_0 \nonumber \\
    &=
    \pi_r(\bx_0)
    \left(
    \int p_{s_r}(\bz_r\mid \bx_0)\,d\bz_r
    \right)
    d\bx_0
    =
    \pi_r(d\bx_0).
    \label{eq:appB-reconstruction-maps-eta-to-pi}
\end{align}
Thus exact reconstruction maps the noisy marginal \(\eta_r\) to the clean target \(\pi_r\).

Next define the Boltzmann--Gibbs reweighting operator
\begin{equation}
    \Psi_{G_r}(\mu)(d\bx_0)
    :=
    \frac{
        G_r(\bx_0)\mu(d\bx_0)
    }{
        \mu(G_r)
    } .
    \label{eq:appB-bg-operator}
\end{equation}
Applying this operator to \(\pi_r\) gives
\begin{align}
    \Psi_{G_r}(\pi_r)(d\bx_0)
    &\propto
    p(\by\mid \bx_0)^{\lambda_{r-1}-\lambda_r}
    p(\bx_0)p(\by\mid \bx_0)^{\lambda_r}
    d\bx_0 \nonumber \\
    &=
    p(\bx_0)p(\by\mid \bx_0)^{\lambda_{r-1}}
    d\bx_0
    \propto
    \pi_{r-1}(d\bx_0).
    \label{eq:appB-bg-maps-pi}
\end{align}
Therefore,
\begin{equation}
    \Psi_{G_r}(\pi_r)=\pi_{r-1}.
    \label{eq:appB-bg-exact}
\end{equation}

Finally, let
\begin{equation}
    P_{r-1}(d\bz_{r-1}\mid \bx_0)
    :=
    p_{s_{r-1}}(d\bz_{r-1}\mid \bx_0)
    \label{eq:appB-propagation-kernel}
\end{equation}
be the forward noising kernel to the next auxiliary noise level. By the definition of \(\eta_{r-1}\),
\begin{equation}
    \pi_{r-1}P_{r-1}
    =
    \eta_{r-1}.
    \label{eq:appB-propagation-maps-pi-to-eta}
\end{equation}

Combining
\eqref{eq:appB-reconstruction-maps-eta-to-pi},
\eqref{eq:appB-bg-exact}, and
\eqref{eq:appB-propagation-maps-pi-to-eta}, one ideal TGD stage implements the exact population flow
\begin{equation}
    \eta_r
    \xrightarrow{K_r}
    \pi_r
    \xrightarrow{\Psi_{G_r}}
    \pi_{r-1}
    \xrightarrow{P_{r-1}}
    \eta_{r-1}.
    \label{eq:appB-population-flow}
\end{equation}
At the terminal stage, there is no additional reweighting or propagation; exact reconstruction gives
\begin{equation}
    \eta_0
    \xrightarrow{K_0}
    \pi_0 .
    \label{eq:appB-terminal-flow}
\end{equation}

\paragraph{Step 2: The particle update consistently approximates this flow.}

Assume that the weighted empirical measure
\begin{equation}
    \widehat{\eta}_r^N
    :=
    \sum_{i=1}^N w_r^{(i)}\delta_{\bz_r^{(i)}}
    \label{eq:appB-eta-empirical}
\end{equation}
is consistent for \(\eta_r\), meaning that for every bounded measurable \(\varphi\),
\[
    \widehat{\eta}_r^N(\varphi)
    \xrightarrow{\mathbb{P}}
    \eta_r(\varphi).
\]

The ideal TGD update first samples
\[
    \bx_0^{(i)} \sim K_r(\cdot\mid \bz_r^{(i)}).
\]
By standard consistency results for sequential Monte Carlo samplers
\cite{delmoral2004feynman, doucet2001sequential, smc_elements},
mutation by a Markov kernel preserves consistency. Hence the resulting weighted empirical measure over clean particles is consistent for
\[
    \eta_rK_r=\pi_r.
\]

TGD then applies the incremental potential
\[
    G_r(\bx_0)
    =
    p(\by\mid \bx_0)^{\lambda_{r-1}-\lambda_r},
\]
using the weight update
\begin{equation}
    \widetilde w_{r-1}^{(i)}
    \propto
    w_r^{(i)}G_r(\bx_0^{(i)})
    =
    w_r^{(i)}
    p(\by\mid \bx_0^{(i)})^{\lambda_{r-1}-\lambda_r}.
    \label{eq:appB-weight-update}
\end{equation}
This is the same incremental weight derived in Appendix~\ref{app:smc-weight-derivations}. Since \(G_r\) is bounded and \(\pi_r(G_r)>0\), standard SMC consistency results imply that weighting preserves consistency, and the weighted clean empirical measure is consistent for
\[
    \Psi_{G_r}(\pi_r)=\pi_{r-1}.
\]

If resampling is performed, the assumed consistency of the resampling scheme preserves this limit. If resampling is omitted, the weighted particle system is left unchanged and consistency is immediate.

Finally, TGD samples
\[
    \bz_{r-1}^{(i)}
    \sim
    P_{r-1}(\cdot\mid \bx_0^{(i)}).
\]
A second application of mutation consistency gives a weighted empirical measure over noisy auxiliary particles that is consistent for
\[
    \pi_{r-1}P_{r-1}=\eta_{r-1}.
\]
Thus, consistency for \(\eta_r\) implies consistency for \(\eta_{r-1}\).

At the terminal stage \(r=0\), if
\[
    \widehat{\eta}_0^N
    =
    \sum_{i=1}^N w_0^{(i)}\delta_{\bz_0^{(i)}}
\]
is consistent for \(\eta_0\), then sampling
\[
    \bx_0^{(i)}\sim K_0(\cdot\mid \bz_0^{(i)})
\]
produces a weighted clean empirical measure consistent for
\[
    \eta_0K_0=\pi_0.
\]

\paragraph{Step 3: Induction over outer stages.}

Because \(\lambda_R=0\), the initial clean target is the prior:
\[
    \pi_R(d\bx_0)=p(\bx_0)d\bx_0.
\]
Therefore the corresponding noisy marginal is
\[
    \eta_R(d\bz_R)
    =
    \int p(\bx_0)p_{s_R}(d\bz_R\mid \bx_0)\,d\bx_0
    =
    p_{s_R}(d\bz_R),
\]
which is exactly the noisy prior used to initialize TGD. Hence the initial empirical measure is consistent for \(\eta_R\).

Applying Step~2 for
\[
    r=R,R-1,\dots,1
\]
shows by induction that the particle system is consistent for \(\eta_0\). Applying the terminal reconstruction kernel \(K_0\) then yields a weighted empirical measure over clean particles that is consistent for
\[
    \pi_0(d\bx_0)=p(\bx_0\mid \by)d\bx_0.
\]
Consequently, for every bounded measurable \(h\),
\[
    \sum_{i=1}^N w_0^{(i)}h(\bx_0^{(i)})
    \xrightarrow{\mathbb{P}}
    \pi_0(h)
    =
    \mathbb{E}_{p(\bx_0\mid \by)}[h(\bx_0)].
\]
This proves the claim.
\end{proof}

\section{Hybrid Conditional Reconstruction and Clean-Space Correction}
\label{app:hybrid-conditional}

Section~\ref{sec:reconstruction-modules} described two practical families of reconstruction modules for approximating the ideal stagewise law
\begin{equation}
    p_{\lambda_r,s_r}(\bx_0\mid \bz_r,\by)
    \propto
    p_{s_r}(\bx_0\mid \bz_r)\,p(\by\mid \bx_0)^{\lambda_r},
    \label{eq:hybrid-stagewise-law}
\end{equation}
where \(\bz_r\) is the auxiliary state at noise level \(s_r\). Conditional reverse-dynamics modules, such as DPS or MPGD solvers, attempt to impose the observation during the reconstruction dynamics. Clean-space correction modules, such as DAPS solvers, first reconstruct using the unconditional diffusion prior and then correct the clean sample using the observation. We show that these two strategies can be viewed as endpoints of a common hybrid construction.

For compactness, write
\[
    \ell(\bx_0) := p(\by\mid \bx_0).
\]
For any intermediate conditioning level \(\nu\in[0,\lambda_r]\), define the weaker stagewise conditional law
\begin{equation}
    p_{\nu,s_r}(\bx_0\mid \bz_r,\by)
    \propto
    p_{s_r}(\bx_0\mid \bz_r)\,\ell(\bx_0)^\nu .
    \label{eq:hybrid-weak-law}
\end{equation}
Then the full stagewise law admits the factorization
\begin{align}
    p_{\lambda_r,s_r}(\bx_0\mid \bz_r,\by)
    &\propto
    p_{s_r}(\bx_0\mid \bz_r)\,\ell(\bx_0)^{\lambda_r} \nonumber \\
    &=
    p_{s_r}(\bx_0\mid \bz_r)\,\ell(\bx_0)^\nu
    \ell(\bx_0)^{\lambda_r-\nu} \nonumber \\
    &\propto
    p_{\nu,s_r}(\bx_0\mid \bz_r,\by)\,
    \ell(\bx_0)^{\lambda_r-\nu}.
    \label{eq:hybrid-factorization}
\end{align}
Thus, one may first approximately sample from a weaker conditional law and then apply a clean-space correction that accounts for the remaining likelihood exponent.

This suggests the following hybrid reconstruction module at outer stage \(r\).

\paragraph{Conditional reverse-dynamics warm start.}
Choose an intermediate conditioning level \(\nu\in[0,\lambda_r]\). Starting from the auxiliary noisy state \(\bz_r\), run a conditional reverse-dynamics reconstruction module whose target is the weaker law
\[
    p_{\nu,s_r}(\bx_0\mid \bz_r,\by)
    \propto
    p_{s_r}(\bx_0\mid \bz_r)\,\ell(\bx_0)^\nu .
\]
Let the resulting clean candidate be denoted by \(\widetilde{\bx}_0\). In practice, this step is approximate, because conditional reverse dynamics use tractable approximations to the intractable conditional score.

\paragraph{Clean-space correction.}
Starting from \(\widetilde{\bx}_0\), apply a clean-space Markov kernel \(M_{r,\nu}\) whose invariant distribution is the full stagewise law
\[
    p_{\lambda_r,s_r}(\bx_0\mid \bz_r,\by).
\]
Examples include Metropolis--Hastings, Langevin, proximal, or other observation-correction kernels, depending on the likelihood model. By the factorization in \eqref{eq:hybrid-factorization}, this correction can be interpreted as accounting for the residual likelihood factor
\[
    \ell(\bx_0)^{\lambda_r-\nu}.
\]
More precisely, the full target is still \(p_{\lambda_r,s_r}(\bx_0\mid \bz_r,\by)\); the residual factor describes how this target differs from the weaker law \(p_{\nu,s_r}(\bx_0\mid \bz_r,\by)\). In the special case where exact samples from \(p_{\nu,s_r}\) are used as an independence proposal, the Metropolis--Hastings acceptance ratio depends only on the residual likelihood factor. For generic local correction kernels, the transition must target the full stagewise density.

The endpoints of this construction recover the two reconstruction families used in the main text. If \(\nu=0\), then
\[
    p_{0,s_r}(\bx_0\mid \bz_r,\by)
    =
    p_{s_r}(\bx_0\mid \bz_r),
\]
so the first step is unconditional reconstruction followed by clean-space observation correction, corresponding to a DAPS-style module. If \(\nu=\lambda_r\), the residual factor is constant, and the clean-space correction step is unnecessary when the conditional reverse-dynamics reconstruction is exact. This recovers a fully conditional reverse-dynamics module. Intermediate values \(0<\nu<\lambda_r\) yield hybrid modules in which conditional reverse dynamics absorb part of the observation likelihood and the clean-space correction accounts for the remaining conditioning.

This hybrid construction is useful conceptually because it shows that conditional reverse dynamics and clean-space correction are compatible within the same TGD reconstruction interface. It also separates two sources of approximation: the accuracy of the weak conditional reconstruction and the mixing or accuracy of the clean-space correction kernel. In our preliminary experiments, however, hybrid choices with \(0<\nu<\lambda_r\) did not yield consistent improvements over the simpler endpoint modules used in the main experiments. We therefore report the endpoint variants in the main paper and include the hybrid construction here for completeness.

\section{Recovering Existing Training-Free Conditional Samplers}
\label{app:recovering-existing-training-free-conditional-samplers}

\begin{figure}
    \centering
    \includegraphics[width=\linewidth]{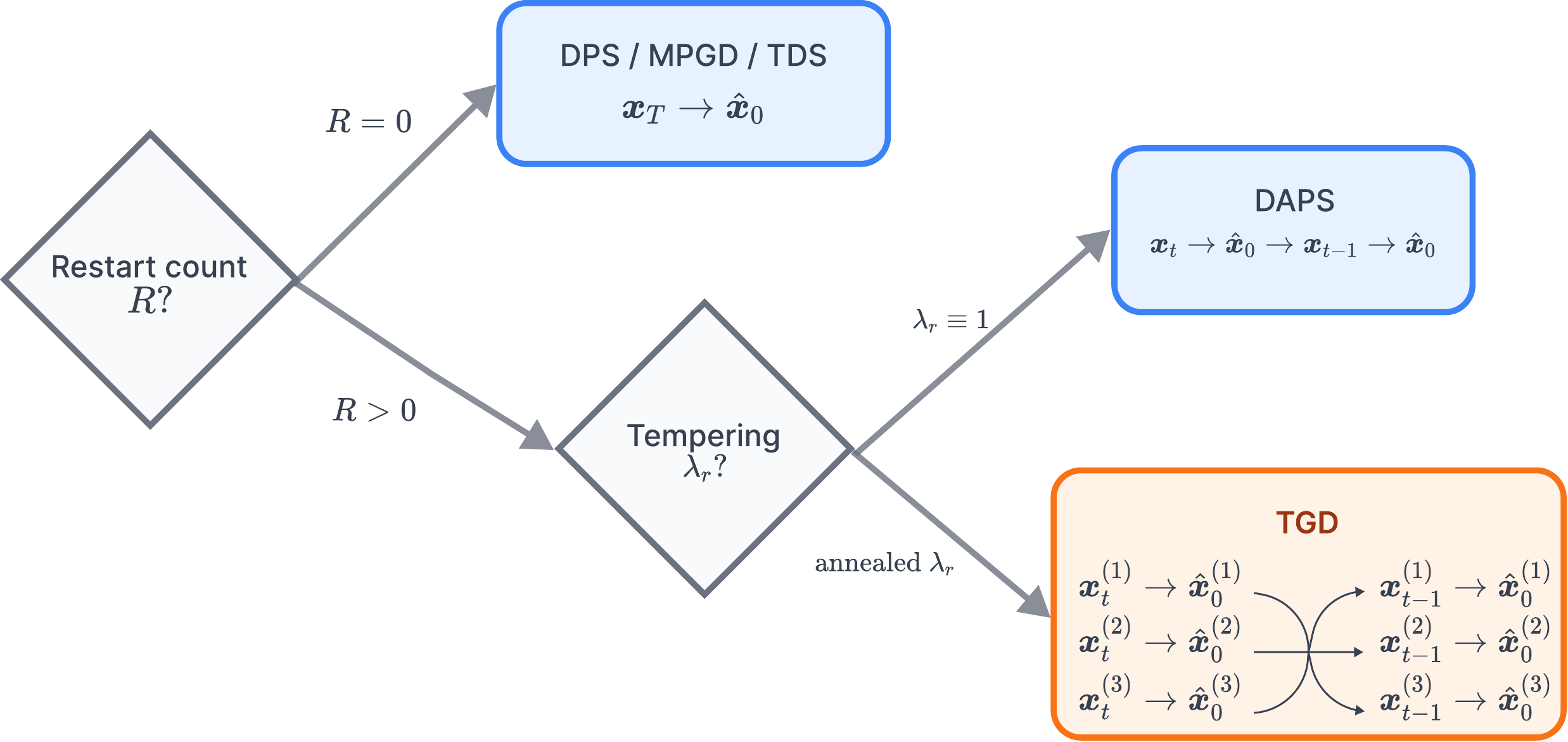}
    \caption{Settings under which existing training-free conditional samplers are recovered as special cases of TGD. Single-stage TGD recovers the chosen conditional reconstruction module, while fully conditioned multi-stage TGD with constant tempering recovers DAPS-style repeated reconstruction and correction.}
    \label{fig:recovery}
\end{figure}

We describe how several existing training-free conditional diffusion samplers are recovered as special cases of TGD. In the notation of Sections~\ref{sec:annealed-clean-targets}--\ref{sec:reconstruction-modules}, TGD is specified by an outer noise schedule \(\{s_r\}_{r=0}^R\), a tempering schedule \(\{\lambda_r\}_{r=0}^R\), a number of outer particles \(N\), and a stagewise reconstruction module approximating
\begin{equation}
    p_{\lambda_r,s_r}(\bx_0\mid \bz_r,\by)
    \propto
    p_{s_r}(\bx_0\mid \bz_r)\,p(\by\mid \bx_0)^{\lambda_r}.
    \label{eq:tgd-special-cases-reconstruction-law}
\end{equation}
The corresponding clean-space annealed targets are
\begin{equation}
    \pi_r(\bx_0)
    \propto
    p(\bx_0)\,p(\by\mid \bx_0)^{\lambda_r},
    \qquad
    0\leq \lambda_R\leq \cdots \leq \lambda_0=1,
    \label{eq:tgd-special-cases-targets}
\end{equation}
and the outer incremental weight update is
\begin{equation}
    \widetilde w_{r-1}^{(i)}
    =
    w_r^{(i)}
    p(\by\mid \bx_0^{(i)})^{\lambda_{r-1}-\lambda_r},
    \qquad
    r=R,\dots,1.
    \label{eq:tgd-special-cases-weights}
\end{equation}
Here \(R\) denotes the number of outer transitions, while \(r=0\) denotes the terminal fully conditioned reconstruction stage.

\paragraph{Single-stage conditional samplers.}
When \(R=0\), TGD has no outer reweighting, resampling, or propagation stages. The tempering schedule consists only of the terminal level \(\lambda_0=1\), and the algorithm initializes
\[
    \bz_0\sim p_{s_0}.
\]
With the standard choice \(s_0=S\), this gives a single fully conditioned reconstruction from the highest-noise prior state:
\begin{equation}
    \bx_0\sim p_{\lambda_0,s_0}(\bx_0\mid \bz_0,\by)
    =
    p_{1,S}(\bx_0\mid \bz_0,\by).
    \label{eq:single-stage-reconstruction}
\end{equation}
Therefore, TGD with \(R=0\) exactly recovers any chosen single-call training-free conditional sampler when that sampler is used as the reconstruction module. Choosing a DPS-style reconstruction module recovers DPS; choosing an MPGD-style reconstruction module recovers MPGD; choosing a TDS reconstruction module recovers TDS. In the TDS case, the particles and weights internal to TDS are part of the reconstruction module itself and are distinct from the outer TGD particle system.

\paragraph{DAPS-style multi-stage reconstruction.}
DAPS-style methods use repeated reconstruction and observation correction across a decreasing sequence of noise levels. This behavior is recovered by choosing a multi-stage noise schedule
\begin{equation}
    S=s_R\geq s_{R-1}\geq \cdots \geq s_0>0,
    \qquad R\geq 1,
    \label{eq:daps-noise-schedule}
\end{equation}
and setting the tempering schedule to be fully conditioned at every stage:
\begin{equation}
    \lambda_r\equiv 1,
    \qquad r=0,\dots,R.
    \label{eq:daps-tempering-schedule}
\end{equation}
Then, for every nonterminal stage \(r=R,\dots,1\),
\[
    \lambda_{r-1}-\lambda_r = 0,
\]
so the incremental potential in \eqref{eq:tgd-special-cases-weights} is constant. Consequently, the outer weighting step is vacuous. If resampling is omitted, and the reconstruction module at each stage is chosen to be the DAPS-style module consisting of unconditional reconstruction followed by clean-space observation correction, then TGD exactly recovers DAPS along the same noise schedule \(\{s_r\}_{r=0}^R\).

For \(N=1\), this gives a single DAPS trajectory. For \(N>1\), omitting resampling gives independent multi-trajectory DAPS sampling. This corresponds to best-of-\(N\) baselines, which run several independent trajectories and select the final reconstruction using a data-consistency score. In contrast, TGD with a nonconstant tempering schedule introduces nontrivial likelihood-ratio weights and optional resampling, allowing outer particles to interact before the terminal reconstruction.

These reductions clarify the role of TGD. When the outer structure collapses to a single stage, TGD reduces to the chosen conditional reconstruction solver. When the tempering schedule is constant and fully conditioned, the likelihood-ratio weights are constant and the method reduces to repeated reconstruction and correction without outer particle interaction. The distinctive TGD regime is the annealed multi-stage setting: outer particles are reconstructed under progressively stronger conditioning, reweighted by likelihood-ratio increments, optionally resampled, and propagated across noise levels. This is the regime that allows computation to concentrate on trajectories that remain plausible under both the diffusion prior and the observation.

\section{Accelerated TGD pseudocode}
\label{app:atgd_pseudocode}

Algorithm~\ref{alg:atgd} uses the following terminal-selection convention. At $\rho=1$, all $N$ particles are propagated through every outer stage and reconstructed at the terminal level. The candidate with highest observation likelihood is then returned, equivalently the candidate with lowest measurement error under the Gaussian observation model. With the same reconstruction module, schedules, and resampling policy, this corresponds to full, unpruned TGD followed by terminal selection for a single reconstruction. This selected reconstruction is distinct from the weighted particle approximation returned by Algorithm~1.

\begin{algorithm}[t]
\caption{Accelerated Tempered Guided Diffusion (A-TGD)}
\label{alg:atgd}
\begin{algorithmic}[1]
\Require observation $\by$,  number of particles $N$, tempering schedule $\{\lambda_r\}_{r=0}^R$, noise schedule $\{s_r\}_{r=0}^R$, conditional reconstruction law $p_{\lambda_r,s_r}(\bx_0 \mid \bz_r,\by)$, pruning factor $\rho \in [0,1]$
\State $K_\rho \gets \min\{R,\max\{1,\lceil \rho R\rceil\}\}$
\State Sample $\bz_R^{(i)} \sim p_{s_R}(\bz_R)$ and set $w_R^{(i)} \gets 1/N$, for $i=1,\dots,N$

\For{$r = R, R-1, \dots, R-K_\rho+1$}
    \For{$i = 1,\dots,N$}
        \State Sample $\bx_0^{(i)} \sim p_{\lambda_r,s_r}(\bx_0 \mid \bz_r^{(i)}, \by)$
        \State $\tilde w_{r-1}^{(i)} \gets w_r^{(i)} p(\by \mid \bx_0^{(i)})^{\lambda_{r-1}-\lambda_r}$
    \EndFor
    \State Normalize $\{\tilde w_{r-1}^{(i)}\}_{i=1}^N$ to obtain $\{w_{r-1}^{(i)}\}_{i=1}^N$
    \State Optionally resample according to $\{w_{r-1}^{(i)}\}_{i=1}^N$ and reset weights to $1/N$
    \For{$i = 1,\dots,N$}
        \State Sample $\bz_{r-1}^{(i)} \sim p_{s_{r-1}}(\bz_{r-1}\mid \bx_0^{(i)})$
    \EndFor
\EndFor

\Statex \textit{// Prune current particle population}
\For{$i = 1,\dots,N$}
    \State Sample $\bar{\bx}_0^{(i)} \sim
    p_{\lambda_{R-K_\rho},s_{R-K_\rho}}
    (\bx_0 \mid \bz_{R-K_\rho}^{(i)},\by)$
\EndFor
\State $i^\star \gets
\arg\max_i p(\by \mid \bar{\bx}_0^{(i)})$
\State Set $\bz_{R-K_\rho} \gets \bz_{R-K_\rho}^{(i^\star)}$

\For{$r = R-K_\rho, R-K_\rho-1, \dots, 1$}
    \State Sample $\bx_0 \sim p_{\lambda_r,s_r}(\bx_0 \mid \bz_r,\by)$
    \State Sample $\bz_{r-1} \sim p_{s_{r-1}}(\bz_{r-1}\mid \bx_0)$
\EndFor

\State Sample $\bx_0 \sim p_{\lambda_0,s_0}(\bx_0 \mid \bz_0,\by)$
\State \Return $\bx_0$
\end{algorithmic}
\end{algorithm}

\section{Experimental and Implementation Details}
\label{app:details}

\subsection{Controlled inverse problems}
\label{app:2d_details}

We use a controlled two-dimensional inverse problem in which the prior, score, likelihood, and
posterior reference sampler are all available analytically. This lets us evaluate the outer particle
procedure without learned-score error.

\paragraph{Prior and analytic score.}
The clean prior is an equally weighted mixture of \(5\) Gaussians in \(\mathbb R^2\). The component
means are sampled uniformly from \([-0.9,0.9]^2\), corresponding to a margin parameter \(0.10\),
and each component has covariance \(\tau^2 I\) with \(\tau=0.005\). We use an analytic
diffusion model: under additive Gaussian noising at noise level \(s\), the noised marginal remains
a Gaussian mixture with component covariances \(\tau^2 I+s^2 I\), and the score is computed
exactly from this noised mixture distribution.

\paragraph{Observation model and posterior reference samples.}
For each of \(10\) test conditions, we sample a clean point \(\bx_0\) from the mixture prior and
generate an elementwise absolute-value observation
\[
    \by = |\bx_0| + \boldsymbol{\epsilon},
    \qquad
    \boldsymbol{\epsilon}\sim \mathcal N(0,0.01^2 I).
\]
The same observations are used for all methods. For each observation, we generate
\(10{,}000\) reference posterior samples from the exact conditional distribution
\(p(\bx_0\mid \by)\). This sampler exploits the diagonal Gaussian-mixture prior and the separable
absolute-value likelihood: it first samples the mixture component conditional on \(\by\), then samples
each coordinate from the corresponding one-dimensional conditional distribution, which is a mixture
over the two possible signs.

\paragraph{Samplers.}
All methods use the same analytic prior score and the same DPS-style conditional drift inside their
reconstruction solver. The proposal likelihood used for DPS-style guidance has standard deviation
\[
    \sigma_{\mathrm{prop}}(s) = 0.8s + 0.01,
\]
while TGD particle weights use the true measurement likelihood with noise standard deviation
\(0.01\).

TGD uses a \(20\)-point EDM outer noise grid with highest noise \(s_R=80\), lowest positive
noise \(s_0=0.002\), and EDM curvature parameter \(7\). Its likelihood tempering schedule is
uniform from \(\lambda_R=0\) to \(\lambda_0=1\). Each outer stage uses a one-step Euler-discretized DPS-style
ODE reconstruction proposal, followed by likelihood weighting, systematic resampling
after nontrivial weighting steps, and EDM re-noising.

We compare against two baselines implemented in the same sampler framework. DPS uses a single
outer stage and a \(20\)-step Euler ODE DPS-style reconstruction trajectory, with no outer particle
interaction. DPS-DAPS uses the same \(20\)-point outer EDM noise grid as TGD but sets
\(\lambda_r=1\) at every stage. Since the tempering exponent is constant, the
incremental likelihood weights are constant, so no resampling is triggered. This
baseline corresponds to repeated DPS/DAPS-style reconstruction without annealed
SMC weighting or particle interaction.

\paragraph{Particle sweep and metrics.}
We sweep the logical particle count
\[
    N\in\{1,2,4,8,16,32,64,128\}.
\]
For each method and particle count, we run the corresponding \(N\)-particle sampler repeatedly and
pool \(10{,}000\) generated samples per observation before computing distributional metrics. For
computational efficiency, the \(N=1\) pooled-sample setting is generated in larger independent
batches, but it is evaluated as the logical single-particle setting.

We report the max-sliced Wasserstein distance, labeled as SWD in
Figure~\ref{fig:toy2d_results}, between generated samples and exact posterior reference samples.
The metric is computed with the POT implementation~\citep{flamary2021pot,flamary2024pot} per
observation using \(100\) random projections, and then summarized over the \(10\) test conditions.
Figure~\ref{fig:toy2d_results} reports means over the \(10\) test conditions, with bands
corresponding to one standard error.

\paragraph{Visualization.}
For Figure~\ref{fig:toy2d_results}(a), we visualize one representative test condition. Prior
contours are estimated from \(20{,}000\) prior samples, and posterior contours are estimated from
exact posterior reference samples for the same observation. We overlay the final \(N=64\) TGD
particles, the ground-truth clean point, and the sign-ambiguous candidate locations induced by the
absolute-value observation.

\paragraph{Twenty-dimensional experiment.}
We extend the Gaussian-mixture inverse problem above to clean signals $x_0\in\mathbb{R}^{20}$, retaining five diagonal Gaussian components and the separable absolute-value observation model. Exact-TGD uses exact draws from the stagewise reconstruction law~(7), rather than the DPS approximation. The baselines use the analytic prior score.

We evaluate $N\in\{4,8,16,32,64,128\}$ over ten test conditions. For each $N$, repeated independent $N$-particle runs are pooled to obtain 10,000 samples per condition; $N$ remains the actual interacting particle count. We report mean marginal $W_1$, obtained by averaging one-dimensional Wasserstein distances over coordinates, and max-sliced $W_1$, computed using 1,000 fixed random projections shared across methods. Component TV measures error in the mixture-component probabilities. Distributional errors are summarized as means and standard errors across conditions.

The Monte Carlo reference error compares independent sets of 10,000 and 20,000 exact-posterior samples per condition. For Exact-TGD and TDS, normalized effective sample size, $\mathrm{ESS}/N$, is measured immediately before resampling. We report its minimum stagewise mean.

\begin{table}[h]
\centering
\small
\caption{Posterior diagnostics for the twenty-dimensional inverse problem. Distributional errors are reported
as mean $\pm$ standard error across 10 conditions. The reference floor compares
10,000 and 20,000 independent exact-posterior draws per condition. ESS is the
minimum-stage mean pre-resampling $\mathrm{ESS}/N$.}
\resizebox{\linewidth}{!}{%
\begin{tabular}{lccccc}
\toprule
Method & $N$
& Mean marginal $W_1\downarrow$
& Max-sliced $W_1\downarrow$
& Component TV$\downarrow$
& Min.\ $\mathrm{ESS}/N$ \\
\midrule
Reference floor & --
& $6.734{\times}10^{-5}\!\pm1.21{\times}10^{-6}$
& $1.546{\times}10^{-4}\!\pm6.12{\times}10^{-6}$
& $0$ & -- \\

Exact-TGD & 4
& $8.675{\times}10^{-5}\!\pm3.06{\times}10^{-6}$
& $2.323{\times}10^{-4}\!\pm1.10{\times}10^{-5}$
& $0$ & $.315$ \\

Exact-TGD & 16
& $7.225{\times}10^{-5}\!\pm1.22{\times}10^{-6}$
& $1.783{\times}10^{-4}\!\pm8.96{\times}10^{-6}$
& $0$ & $.201$ \\

Exact-TGD & 128
& $7.270{\times}10^{-5}\!\pm2.46{\times}10^{-6}$
& $1.742{\times}10^{-4}\!\pm9.98{\times}10^{-6}$
& $0$ & $.197$ \\

TDS & 4
& $.1764\!\pm.0126$
& $.6017\!\pm.0521$
& $.2722\!\pm.0148$
& $.456$ \\

TDS & 16
& $.00617\!\pm.00120$
& $.02141\!\pm.00459$
& $.00832\!\pm.00173$
& $.296$ \\

TDS & 128
& $7.718{\times}10^{-4}\!\pm2.44{\times}10^{-5}$
& $2.038{\times}10^{-3}\!\pm1.08{\times}10^{-4}$
& $0$
& $.145$ \\

DPS & 128
& $.4014\!\pm.0016$
& $1.288\!\pm.028$
& $.656\!\pm.023$
& -- \\

DPS-DAPS & 128
& $.3831\!\pm.0098$
& $1.237\!\pm.050$
& $.612\!\pm.014$
& -- \\
\bottomrule
\end{tabular}%
}
\end{table}

Exact-TGD achieves distributional errors close to the Monte Carlo reference error. Although its minimum normalized ESS decreases with $N$, the corresponding absolute ESS increases from approximately 1.3 particles at $N=4$ to 25.2 at $N=128$. The posterior is numerically concentrated on one mixture component in this setting, so zero component TV does not establish within-component accuracy; the marginal and max-sliced distances remain informative.

\paragraph{Compute.}
The controlled experiments were run on CPU and completed in a few minutes.

\subsection{Noise and ODE discretization}

We use the EDM noise discretization~\citep{karras2022elucidatingdesignspacediffusionbased},
following DAPS~\citep{Zhang_2025}. Our implementation uses two discretizations. The first is
the outer TGD discretization, which defines the noise levels \(s_R,\ldots,s_0\) at which particles
are reconstructed, likelihood-reweighted, optionally resampled, and re-noised. For all image
experiments, we use the same outer EDM discretization as DAPS, with maximum noise
\(s_{R}=100\), minimum positive noise \(s_{0}=0.1\), and curvature parameter \( 7\).

The second is the inner ODE discretization used inside each reconstruction module. We denote the inner ODE variable by \(\tau\) to distinguish it from the outer TGD levels \(s_r\). All image experiments use four Euler ODE steps per reconstruction module, with minimum positive inner noise \(\tau_{\min}=0.01\) and curvature parameter \( 7 \). 

\subsection{Inverse problem setup}

We use the same inverse-problem definitions and preprocessing as DAPS~\citep{Zhang_2025},
which follows the DPS benchmark setup~\citep{chung2024diffusionposteriorsamplinggeneral}. In all image experiments,
the observation noise standard deviation is \(\sigma=0.05\).

For inpainting, the forward operator \(\mathcal A\) applies the prescribed pixel mask to the clean
image, and the observation is generated as
\[
\by = \mathcal A(\bx_0) + \bm{\epsilon}, \qquad \bm{\epsilon} \sim \mathcal N(\bm{0},\sigma^2 \boldsymbol{I}).
\]
We use the same mask-generation protocol as DAPS and fix the mask realization across methods
for fair comparison.

For phase retrieval, we use the normalized observation model from DAPS:
\[
\by \sim
\mathcal N\!\left(
\left|\bF \bP \left(0.5\bx_0+0.5\right)\right|,
\sigma^2 \boldsymbol{I}
\right),
\]
where \(\bF\) is the discrete Fourier transform and \(\bP\) is the oversampling operator. Following
DAPS, we use oversampling factor \(k=2\) and \(n=8\). The affine transformation
\(0.5\bx_0+0.5\) maps images from \([-1,1]\) to \([0,1]\) before applying the phase-retrieval
operator.

\subsection{Pretrained diffusion models}
\label{app:pretrained_models}

We use the same pretrained diffusion models as DAPS~\citep{Zhang_2025}. For FFHQ, we use the
\(256\times256\) model used in DPS~\citep{chung2024diffusionposteriorsamplinggeneral}.
For ImageNet, we use the \(256\times256\)  model from
\citep{dhariwal2021diffusionmodelsbeatgans}. All image inverse-problem generation and evaluation
protocols follow DAPS~\citep{Zhang_2025}, except that DAPS reports oracle best-of-\(N\) scores using the trajectory with the best LPIPS, PSNR, or SSIM against the ground truth, whereas we select the trajectory with the lowest measurement error. Since the ground-truth image is unavailable in practical inverse problems, oracle LPIPS/PSNR/SSIM selection is not possible at test time, while measurement-error selection depends only on the observation.

\subsection{Existing assets and licenses}
\label{app:asset-licenses}

We use existing datasets, pretrained models, baseline implementations, and evaluation libraries only for research evaluation. FFHQ~\citep{karras2019style} is used under the Creative Commons BY-NC-SA 4.0 license. ImageNet~\citep{deng2009imagenet} is used under the ImageNet terms of access for non-commercial research and educational use. We do not redistribute FFHQ, ImageNet, or pretrained diffusion checkpoints in the supplemental material; the reproduction instructions direct users to obtain these assets from their original sources and follow the corresponding terms of use.

Our implementation also depends on standard open-source scientific-computing and evaluation packages, including PyTorch, LPIPS~\citep{zhang2018unreasonable}, and POT~\citep{flamary2021pot,flamary2024pot}. The supplemental code includes only materials that we are permitted to redistribute, and all third-party assets are credited through the corresponding citations and original access instructions.

\subsection{Reconstruction modules}

We instantiate TGD with DPS-, MPGD-, and DAPS-style reconstruction modules. Each module maps
a noisy auxiliary state \(\bx_s\) at noise level \(s\) to an approximate clean reconstruction, which is
then passed to the outer TGD likelihood-weighting, resampling, pruning, and re-noising steps. For the
DPS- and MPGD-style modules, we use the conditional-score decomposition in~\eqref{eq:conditional-score} and approximate the intractable conditional term using the
denoised estimate \(\hat{\bx}_0(\bx_s,s)\).

For DPS, we use the standard plug-in approximation
\[
\nabla_{\bx_s}\log p_s(\by\mid \bx_s)
\approx
\nabla_{\bx_s}
\log \hat p_s\!\left(\by\mid \hat{\bx}_0(\bx_s,s)\right).
\]
For MPGD, we instead use a clean-space guidance term
\[
\nabla_{\bx_s}\log p_s(\by\mid \bx_s)
\approx
\nabla_{\bx_0}
\log \hat p_s\!\left(\by\mid \hat{\bx}_0(\bx_s,s)\right),
\]
which is inserted into the MPGD reconstruction update. In both cases, the approximate proposal
likelihood used inside the reconstruction module is
\[
\hat p_s(\by\mid \bx)
=
\mathcal N\!\left(
\by;\mathcal A(\bx),(\gamma s+\sigma)^2 \boldsymbol{I}
\right),
\]
where \(\mathcal A\) is the task forward operator, \(\sigma\) is the observation noise, and
\(\gamma\) is a proposal-smoothing parameter that inflates the likelihood variance at high diffusion
noise levels. This smoothed likelihood is used inside the reconstruction proposal. The outer TGD
particle weights use the true measurement likelihood \(p(\by\mid \bx_0)\). For pruning, we select
the particle with smallest measurement-space error,
\[
i^\star
=
\arg\min_i \|\mathcal A(\hat{\bx}_0^{(i)})-\by\|_2^2,
\]
which is equivalent to maximum likelihood selection under a Gaussian observation model with fixed
variance. We use \(\sigma=0.05\) in all image experiments, with \(\gamma=0.7\) for inpainting and
\(\gamma=0.4\) for phase retrieval.

For DAPS-style reconstruction, we follow the DAPS implementation~\citep{Zhang_2025}. In
particular, DAPS uses 100 Langevin correction steps with a linearly decayed step size and the
correction-noise setting reported in that work. This DAPS-style module is used for the phase
retrieval A-TGD experiments and for the DAPS baselines.

\subsection{Annealing schedules}
\label{app:ann_schedules}

\begin{figure}[t]
    \centering
    \includegraphics[width=0.78\linewidth]{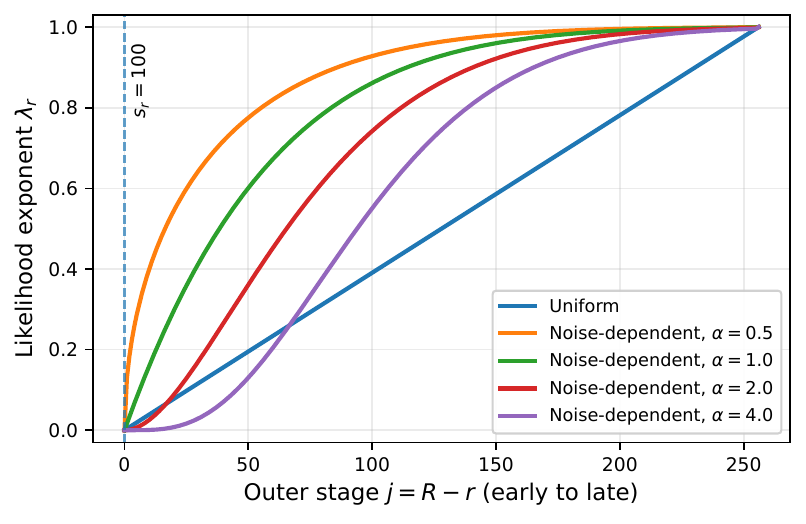}
    \caption{
    Likelihood-tempering schedules for the EDM discretization with 256 steps. The uniform schedule increases the exponent
    linearly across outer stages, while the noise-dependent schedules tie the exponent to the
    EDM noise level \(s_r\).
    }
    \label{fig:annealing-schedules}
\end{figure}

TGD uses a sequence of likelihood-tempering exponents
\[
\lambda_R,\lambda_{R-1},\ldots,\lambda_0
\]
to move from weakly conditioned targets to the final posterior. Stage \(r=R\) is the earliest,
highest-noise stage, while stage \(r=0\) is the final, lowest-noise stage. We require
\[
0 \leq \lambda_R \leq \lambda_{R-1} \leq \cdots \leq \lambda_0 = 1,
\]
so the final target corresponds to the full posterior. 

\paragraph{Uniform schedule.}
The uniform schedule increases the likelihood exponent linearly across outer stages:
\[
\lambda_r
=
\lambda_R
+
\frac{R-r}{R}
\left(1-\lambda_R\right),
\qquad r=0,\ldots,R .
\]
Thus the initial stage uses exponent \(\lambda_R\), while the final stage uses
\(\lambda_0=1\). Setting \(\lambda_R=1\) gives the no-annealing variant, in which
the full likelihood is applied at every stage.

\paragraph{Noise-dependent schedule.}
We also consider a schedule that ties the likelihood exponent directly to the diffusion noise level.
Let \(s_r\) denote the noise level at stage \(r\), ordered as
\[
s_R \geq s_{R-1} \geq \cdots \geq s_0 .
\]
For a curvature parameter \(\alpha>0\), we define the normalized noise progress
\[
u_r
=
\left(
\frac{s_R-s_r}{s_R-0}
\right)^{\alpha},
\]
and set the likelihood exponent to
\[
\lambda_r
=
\lambda_R
+
(1-\lambda_R)
\left(
\frac{s_R-s_r}{s_R-0}
\right)^{\alpha}.
\]
This gives \(u_R=0\) and \(u_0=1\), so the schedule starts from the initial likelihood exponent
\(\lambda_R\) at the highest noise level and reaches the full likelihood exponent \(1\) at the
lowest noise level. The parameter \(\alpha\) controls the curvature of the transition:
\(\alpha=1\) gives a linear transition in the noise level, while larger values keep
\(\lambda_r\) closer to \(\lambda_R\) until lower noise levels. In our experiments, we use
\(s_R=100\). Figure~\ref{fig:annealing-schedules} visualizes the uniform schedule
and several noise-dependent schedules for the EDM discretization with 256 steps.

\subsection{Hyperparameter selection}
\label{app:hyperparameter_selection}

All image-experiment hyperparameters were selected using a small tuning set of five images that
was disjoint from both the diffusion-model training data and the 100-image validation set used for
reporting results. After selecting the hyperparameters on this tuning set, we fixed them for all
reported evaluations. The reported PSNR, SSIM, LPIPS, and runtime values were computed only on
the held-out validation images and were not used for hyperparameter selection.

\subsection{Image experiment configurations}
\label{app:image_configs}

Table~\ref{tab:image-atgd-configs} summarizes the A-TGD configurations used in the main image
experiments. All main image experiments use \(N=4\) initial particles. For inpainting, we use an MPGD-style reconstruction module, uniform likelihood tempering, and systematic resampling after every likelihood-weighting step. For phase retrieval, we use a
DAPS-style reconstruction module and no likelihood tempering, i.e. \(\lambda_r=1\) for all stages;
therefore the incremental weights are constant and no resampling is performed. The pruning factor
is \(\rho=0.5\) for all image experiments except ImageNet phase retrieval, where we use
\(\rho=0.8\).

\begin{table}[t]
\centering
\small
\setlength{\tabcolsep}{4pt}
\caption{
Main A-TGD configurations for image inverse problems.
All settings use \(N=4\) particles. When resampling is enabled, we use systematic resampling.
}
\label{tab:image-atgd-configs}
\begin{tabular}{lllccc}
\toprule
Task & Dataset & Module & \(\rho\) & Annealing & Resampling \\
\midrule
Inpainting & FFHQ     & MPGD & 0.5 & Uniform & Always \\
Inpainting & ImageNet & MPGD & 0.5 & Uniform & Always \\
Phase retrieval & FFHQ     & DAPS & 0.5 & None, \(\lambda_r=1\) & Never \\
Phase retrieval & ImageNet & DAPS & 0.8 & None, \(\lambda_r=1\) & Never \\
\bottomrule
\end{tabular}
\end{table}

When resampling is enabled, we use systematic resampling and reset the particle weights to
\(1/N\). Our implementation also supports effective-sample-size- (ESS) threshold resampling, but we do
not use it in the main experiments. In preliminary experiments, ESS-threshold resampling produced
similar behavior to the simpler always-resample or never-resample policies, while introducing an
additional task-dependent hyperparameter. We therefore use the simpler policies in
Table~\ref{tab:image-atgd-configs} to keep the matched-runtime comparison transparent.

\subsection{Matched-runtime image configurations}

Table~\ref{tab:main-image-configs} reports the sampling configurations used for the matched-runtime
image experiments. For ODE-based methods, the NFE scheme is written as the number of
reconstruction calls times the four ODE steps used per reconstruction. For DPS, which does not use
the restart-based reconstruction interface, we write \(0\times K\), where \(K\) is the total number
of function evaluations. For DAPS \((N=4)\), the NFE scheme is reported per trajectory; the
reported wall-clock time includes all four trajectories. Wall-clock time per image is measured by
dividing the total sampling time for the 100-image evaluation set by 100.

\begin{table}[t]
\centering
\small
\caption{Matched-runtime image sampling configurations.}
\label{tab:main-image-configs}
\begin{tabular}{lllcc}
\toprule
Task & Dataset & Method & NFE scheme & Time/img (s) \\
\midrule

\multirow{8}{*}{Inpainting}
& \multirow{4}{*}{FFHQ}
& DPS & \(0\times596\) & 6.34 \\
& & DAPS \((N=1)\) & \(299\times4\) & 6.44 \\
& & DAPS \((N=4)\) & \(77\times4\) & 6.50 \\
& & A-TGD & \(128\times4\) & 6.47 \\
\cmidrule(lr){2-5}
& \multirow{4}{*}{ImageNet}
& DPS & \(0\times596\) & 23.63 \\
& & DAPS \((N=1)\) & \(320\times4\) & 23.52 \\
& & DAPS \((N=4)\) & \(80\times4\) & 23.84 \\
& & A-TGD & \(128\times4\) & 23.64 \\

\midrule

\multirow{8}{*}{Phase retrieval}
& \multirow{4}{*}{FFHQ}
& DPS & \(0\times404\) & 4.34 \\
& & DAPS \((N=1)\) & \(159\times4\) & 4.30 \\
& & DAPS \((N=4)\) & \(41\times4\) & 4.35 \\
& & A-TGD & \(64\times4\) & 4.35 \\
\cmidrule(lr){2-5}
& \multirow{4}{*}{ImageNet}
& DPS & \(0\times7860\) & 346 \\
& & DAPS \((N=1)\) & \(3855\times4\) & 357 \\
& & DAPS \((N=4)\) & \(1070\times4\) & 352 \\
& & A-TGD & \(1250\times4\) & 345 \\

\bottomrule
\end{tabular}
\end{table}

\subsection{Speed-quality tradeoff configuration}

Figure~\ref{fig:speed-quality} evaluates the speed-quality tradeoff on FFHQ phase retrieval.
For each method, we sweep the sampling budget and record reconstruction quality on 100 validation
images. Wall-clock time per image is measured by dividing the total sampling time for the 100-image
evaluation set by 100, including all independent trajectories or particle computations.

Table~\ref{tab:speed-quality-configs} reports the configurations used to generate the curves.
For ODE-based methods, the NFE scheme is reported as the number of reconstruction calls times
the four ODE steps used per reconstruction. For DPS, which does not use the restart-based
reconstruction interface, we write \(0\times K\) to indicate no reconstruction restarts and \(K\)
total function evaluations.

For Figure~\ref{fig:speed-quality}b, we compare A-TGD and DAPS \((N=4)\) at fixed LPIPS
thresholds
\[
\ell \in \{0.45,0.40,0.35,0.30,0.25,0.20\}.
\]
For each threshold, we estimate the wall-clock time required by each method to reach that LPIPS
value by linearly interpolating its LPIPS--runtime curve from Table~\ref{tab:speed-quality-configs}.
The speedup is
\[
\mathrm{speedup}(\ell)
=
\frac{
T_{\mathrm{DAPS}(N=4)}(\ell)
}{
T_{\mathrm{A\text{-}TGD}}(\ell)
},
\]
where \(T_m(\ell)\) is the interpolated time per image required by method \(m\) to reach LPIPS
threshold \(\ell\). Values greater than one indicate that A-TGD reaches the same LPIPS threshold
faster than independent best-of-four DAPS sampling.

\begin{table}[t]
\centering
\caption{FFHQ phase retrieval speed-quality sweep.}
\label{tab:speed-quality-configs}
\begin{tabular}{lcccc}
\toprule
Method & NFE scheme & Time/img (s) $\downarrow$ & LPIPS $\downarrow$ & PSNR $\uparrow$ \\
\midrule

\multirow{6}{*}{DPS}
& \(0\times68\)   & 0.73 & \ms{0.561}{0.063} & \ms{12.67}{2.42} \\
& \(0\times104\)  & 1.12 & \ms{0.562}{0.063} & \ms{12.67}{2.42} \\
& \(0\times204\)  & 2.18 & \ms{0.563}{0.063} & \ms{12.66}{2.42} \\
& \(0\times404\)  & 4.31 & \ms{0.564}{0.063} & \ms{12.65}{2.42} \\
& \(0\times800\)  & 8.53 & \ms{0.564}{0.063} & \ms{12.65}{2.42} \\
& \(0\times1592\) & 16.97 & \ms{0.564}{0.063} & \ms{12.65}{2.42} \\

\midrule

\multirow{6}{*}{DAPS \((N=1)\)}
& \(26\times4\)  & 0.70 & \ms{0.427}{0.191} & \ms{19.18}{7.40} \\
& \(42\times4\)  & 1.13 & \ms{0.354}{0.223} & \ms{21.95}{8.54} \\
& \(81\times4\)  & 2.18 & \ms{0.245}{0.173} & \ms{25.67}{8.19} \\
& \(159\times4\) & 4.29 & \ms{0.225}{0.185} & \ms{26.98}{8.35} \\
& \(315\times4\) & 8.51 & \ms{0.208}{0.175} & \ms{28.10}{8.08} \\
& \(628\times4\) & 16.94 & \ms{0.225}{0.191} & \ms{27.52}{8.68} \\

\midrule

\multirow{6}{*}{DAPS \((N=4)\)}
& \(7\times4\)   & 0.74 & \ms{0.561}{0.099} & \ms{16.83}{4.08} \\
& \(11\times4\)  & 1.16 & \ms{0.478}{0.128} & \ms{18.89}{4.73} \\
& \(21\times4\)  & 2.21 & \ms{0.359}{0.160} & \ms{22.42}{6.10} \\
& \(41\times4\)  & 4.31 & \ms{0.207}{0.111} & \ms{27.81}{5.12} \\
& \(81\times4\)  & 8.51 & \ms{0.169}{0.116} & \ms{29.56}{5.31} \\
& \(162\times4\) & 17.03 & \ms{0.135}{0.033} & \ms{31.40}{2.38} \\

\midrule

\multirow{6}{*}{A-TGD}
& \(9\times4\)   & 0.71 & \ms{0.508}{0.114} & \ms{18.13}{4.53} \\
& \(16\times4\)  & 1.14 & \ms{0.388}{0.143} & \ms{21.59}{5.38} \\
& \(32\times4\)  & 2.19 & \ms{0.259}{0.148} & \ms{25.62}{6.11} \\
& \(64\times4\)  & 4.30 & \ms{0.172}{0.106} & \ms{29.43}{4.07} \\
& \(128\times4\) & 8.53 & \ms{0.153}{0.089} & \ms{30.24}{4.49} \\
& \(256\times4\) & 16.97 & \ms{0.152}{0.091} & \ms{30.63}{4.60} \\

\bottomrule
\end{tabular}
\end{table}

\section{Ablations}
\label{app:ablations}

Unless otherwise stated, all ablations are run on FFHQ using 100 validation images and report mean \(\pm\) standard deviation for PSNR, SSIM, LPIPS, and wall-clock time per image in seconds. We use \(N=4\), \(\rho=0.5\), $64$ restarts and 4 ODE sampling steps with the MPGD as the default reconstruction module. For inpainting, the default SMC configuration uses the uniform annealing with resampling; for phase retrieval, it uses no annealing. These defaults are used only for the ablations in this section; the main experiments use the task-specific configurations described in Appendix~\ref{app:details}.

\paragraph{Number of particles.}
Table~\ref{tab:particle-ablation} varies the number of initial A-TGD particles \(N\). For FFHQ inpainting, increasing \(N\) has only a small effect on reconstruction quality: performance saturates around \(N=4\), and larger particle counts mainly increase wall-clock time. For FFHQ phase retrieval, the effect of \(N\) is much stronger. Increasing \(N\) from 1 to 4 substantially improves PSNR, SSIM, and LPIPS, and larger values continue to improve quality, although with diminishing returns and increased runtime.

These results indicate that the value of particle exploration depends on the ambiguity of the inverse problem. Inpainting is comparatively stable, so a small particle population is sufficient. Phase retrieval exhibits higher trajectory variability, and therefore benefits more from maintaining multiple early particles before pruning.

\begin{table}[t]
  \caption{Particle-count ablation for A-TGD on FFHQ.}
  \label{tab:particle-ablation}
  \centering
  \small
  \begin{tabular}{llcccc}
    \toprule
    Task & Particles & Time/img (s) $\downarrow$ & PSNR $\uparrow$ & SSIM $\uparrow$ & LPIPS $\downarrow$ \\
    \midrule

    \multirow{5}{*}{Inpainting}
      & $N=1$  & 1.29 & \ms{24.61}{2.70} & \ms{0.846}{0.026} & \ms{0.126}{0.028} \\
      & $N=2$  & 1.92 & \ms{24.73}{2.34} & \ms{0.848}{0.024} & \ms{0.125}{0.028} \\
      & $N=4$  & 3.26 & \bestms{25.01}{2.34} & \bestms{0.849}{0.025} & \ms{0.124}{0.028} \\
      & $N=8$  & 5.85 & \ms{24.88}{2.28} & \ms{0.848}{0.025} & \ms{0.124}{0.028} \\
      & $N=16$ & 11.14 & \ms{24.92}{2.45} & \ms{0.849}{0.025} & \bestms{0.124}{0.028} \\

    \midrule

    \multirow{5}{*}{Phase Retrieval}
      & $N=1$  & 1.30 & \ms{22.26}{7.61} & \ms{0.678}{0.213} & \ms{0.303}{0.194} \\
      & $N=2$  & 1.94 & \ms{24.76}{7.02} & \ms{0.734}{0.192} & \ms{0.248}{0.169} \\
      & $N=4$  & 3.27 & \ms{26.38}{5.63} & \ms{0.778}{0.145} & \ms{0.206}{0.129} \\
      & $N=8$  & 5.91 & \ms{27.00}{5.52} & \ms{0.793}{0.151} & \ms{0.193}{0.125} \\
      & $N=16$ & 11.27 & \bestms{27.24}{4.85} & \bestms{0.798}{0.131} & \bestms{0.189}{0.110} \\

    \bottomrule
  \end{tabular}
\end{table}

\paragraph{Pruning fraction.}
Table~\ref{tab:rho-ablation} varies the pruning factor \(\rho\), which controls how long A-TGD maintains multiple particles before pruning to a single trajectory. For inpainting, quality is relatively insensitive to \(\rho\), and aggressive early pruning already gives strong performance. For phase retrieval, the pruning factor has a much larger effect: later pruning substantially improves PSNR, SSIM, and LPIPS, at the cost of increased runtime.

\begin{table}[t]
  \caption{Pruning-fraction ablation for A-TGD on FFHQ.}
  \label{tab:rho-ablation}
  \centering
  \small
  \begin{tabular}{llcccc}
    \toprule
    Task & Pruning factor $\rho$ & Time/img (s) $\downarrow$ & PSNR $\uparrow$ & SSIM $\uparrow$ & LPIPS $\downarrow$ \\
    \midrule

    \multirow{5}{*}{Inpainting}
      & $0.10$ & 1.76 & \ms{24.88}{2.32} & \ms{0.848}{0.025} & \bestms{0.123}{0.028} \\
      & $0.25$ & 2.30 & \ms{24.65}{2.57} & \ms{0.847}{0.025} & \ms{0.125}{0.028} \\
      & $0.50$ & 3.26 & \bestms{25.01}{2.34} & \bestms{0.849}{0.025} & \ms{0.124}{0.028} \\
      & $0.75$ & 4.22 & \ms{24.64}{2.44} & \ms{0.845}{0.025} & \ms{0.126}{0.027} \\
      & $1.00$ & 5.12 & \ms{24.93}{2.41} & \ms{0.849}{0.024} & \ms{0.124}{0.029} \\

    \midrule

    \multirow{5}{*}{Phase Retrieval}
      & $0.10$ & 1.80 & \ms{23.25}{7.88} & \ms{0.688}{0.225} & \ms{0.292}{0.203} \\
      & $0.25$ & 2.33 & \ms{23.47}{7.82} & \ms{0.697}{0.219} & \ms{0.282}{0.194} \\
      & $0.50$ & 3.28 & \ms{26.38}{5.63} & \ms{0.778}{0.145} & \ms{0.206}{0.129} \\
      & $0.75$ & 4.23 & \ms{27.01}{5.13} & \ms{0.795}{0.145} & \ms{0.193}{0.124} \\
      & $1.00$ & 5.13 & \bestms{27.19}{5.13} & \bestms{0.802}{0.131} & \bestms{0.188}{0.111} \\

    \bottomrule
  \end{tabular}
\end{table}

\paragraph{Annealing and resampling.}
Table~\ref{tab:anneal-resample-ablation} studies the effect of annealing and resampling. For this ablation, we set \(\rho=1.0\) to isolate the effect of annealing and resampling at matched compute across variants. “No annealing” sets \(\lambda_r=1\) for all stages, applying the full likelihood at every stage, while “no resampling” uses the annealed schedule but disables resampling. The results show that the preferred SMC configuration differs across tasks. For inpainting, annealing with resampling gives the best reconstruction quality across PSNR, SSIM, and LPIPS, although the margin over the ablated variants is modest. For phase retrieval, however, the no-annealing variant performs best across all metrics. We therefore use annealing with resampling as the default for FFHQ inpainting and no annealing as the default for FFHQ phase retrieval.

\begin{table}[t]
  \caption{Ablation of annealing and resampling.}
  \label{tab:anneal-resample-ablation}
  \centering
  \small
  \begin{tabular}{llcccc}
    \toprule
    Task & Variant & Time/img (s) $\downarrow$ & PSNR $\uparrow$ & SSIM $\uparrow$ & LPIPS $\downarrow$ \\
    \midrule

    \multirow{3}{*}{Inpainting}
      & No annealing & 5.05 & \ms{24.36}{2.64} & \ms{0.843}{0.027} & \ms{0.130}{0.031} \\
      & No resampling & 5.06 & \ms{24.42}{2.59} & \ms{0.845}{0.027} & \ms{0.127}{0.031} \\
      & Annealing + resampling & 5.05 & \bestms{24.93}{2.41} & \bestms{0.849}{0.024} & \bestms{0.124}{0.029} \\

    \midrule

    \multirow{3}{*}{Phase Retrieval}
      & No annealing & 5.05 & \bestms{27.19}{5.13} & \bestms{0.802}{0.131} & \bestms{0.188}{0.111} \\
      & No resampling & 5.06 & \ms{26.98}{5.23} & \ms{0.796}{0.135} & \ms{0.191}{0.113} \\
      & Annealing + resampling & 5.07 & \ms{24.80}{7.79} & \ms{0.727}{0.217} & \ms{0.254}{0.190} \\

    \bottomrule
  \end{tabular}
\end{table}

\paragraph{Reconstruction module.}
Table~\ref{tab:reconstruction-module-ablation} compares different reconstruction modules inside
the same A-TGD outer procedure. We use no annealing in this ablation. We use MPGD as the default reconstruction module because it provides the best wall-clock efficiency while remaining competitive in reconstruction quality. The best-quality module, however, depends on the task. For inpainting, DPS gives the strongest PSNR, SSIM, and LPIPS, but is substantially slower. For phase retrieval, DAPS gives the best reconstruction quality, while MPGD is faster and remains competitive.

\begin{table}[t]
  \caption{Reconstruction-module ablation inside the A-TGD outer loop.}
  \label{tab:reconstruction-module-ablation}
  \centering
  \small
  \begin{tabular}{llcccc}
    \toprule
    Task & Reconstruction module & Time/img (s) $\downarrow$ & PSNR $\uparrow$ & SSIM $\uparrow$ & LPIPS $\downarrow$ \\
    \midrule

    \multirow{3}{*}{Inpainting}
      & DPS  & 6.75 & \bestms{25.66}{2.42} & \bestms{0.873}{0.023} & \bestms{0.112}{0.024} \\
      & MPGD & \textbf{3.23} & \ms{24.86}{2.35} & \ms{0.847}{0.025} & \ms{0.126}{0.028} \\
      & DAPS & 3.41 & \ms{24.60}{2.00} & \ms{0.791}{0.027} & \ms{0.156}{0.034} \\

    \midrule

    \multirow{3}{*}{Phase Retrieval}
      & DPS  & 6.77 & \ms{26.38}{5.87} & \ms{0.782}{0.151} & \ms{0.217}{0.129} \\
      & MPGD & \textbf{3.27} & \ms{26.38}{5.63} & \ms{0.778}{0.145} & \ms{0.206}{0.129} \\
      & DAPS & 4.27 & \bestms{29.43}{4.07} & \bestms{0.811}{0.104} & \bestms{0.172}{0.106} \\

    \bottomrule
  \end{tabular}
\end{table}

The ablations show that A-TGD is most useful when early trajectories differ substantially in quality. Inpainting is relatively stable and benefits little from larger particle populations or later pruning, whereas phase retrieval benefits strongly from both. The reconstruction-module ablation confirms that the outer A-TGD procedure is modular: MPGD provides the fastest default, while DPS or DAPS can improve quality for specific tasks at higher cost. The annealing ablation shows that tempering is not uniformly beneficial: it helps modestly for inpainting but hurts phase retrieval in this setting.

\begin{table}[t]
  \caption{Annealing-schedule ablation for A-TGD on FFHQ without resampling.}
  \label{tab:annealing-ablation-no-resampling}
  \centering
  \small
  \begin{tabular}{llcccc}
    \toprule
    Task & Annealing schedule & Time/img (s) $\downarrow$ & PSNR $\uparrow$ & SSIM $\uparrow$ & LPIPS $\downarrow$ \\
    \midrule

    \multirow{6}{*}{Inpainting}
      & No annealing & 3.29 & \ms{24.86}{2.35} & \ms{0.847}{0.025} & \ms{0.126}{0.028} \\
      & Uniform annealing & 3.27 & \ms{24.85}{2.34} & \bestms{0.848}{0.025} & \bestms{0.124}{0.028} \\
      & Noise-dependent, $\alpha=0.5$ & 3.27 & \bestms{24.87}{2.27} & \ms{0.847}{0.024} & \ms{0.126}{0.028} \\
      & Noise-dependent, $\alpha=1$ & 3.27 & \ms{24.85}{2.31} & \ms{0.847}{0.024} & \ms{0.126}{0.028} \\
      & Noise-dependent, $\alpha=2$ & 3.27 & \ms{24.86}{2.25} & \ms{0.847}{0.023} & \ms{0.126}{0.027} \\
      & Noise-dependent, $\alpha=4$ & 3.27 & \ms{24.82}{2.51} & \ms{0.847}{0.025} & \ms{0.126}{0.029} \\

    \midrule

    \multirow{6}{*}{Phase Retrieval}
      & No annealing & 3.25 & \bestms{26.38}{5.63} & \bestms{0.778}{0.145} & \bestms{0.206}{0.129} \\
      & Uniform annealing & 3.24 & \ms{22.92}{6.45} & \ms{0.689}{0.189} & \ms{0.284}{0.163} \\
      & Noise-dependent, $\alpha=0.5$ & 3.24 & \ms{26.12}{5.80} & \ms{0.770}{0.153} & \ms{0.212}{0.133} \\
      & Noise-dependent, $\alpha=1$ & 3.23 & \ms{26.01}{5.77} & \ms{0.770}{0.150} & \ms{0.211}{0.132} \\
      & Noise-dependent, $\alpha=2$ & 3.24 & \ms{25.33}{6.47} & \ms{0.750}{0.173} & \ms{0.230}{0.155} \\
      & Noise-dependent, $\alpha=4$ & 3.23 & \ms{25.03}{6.57} & \ms{0.747}{0.179} & \ms{0.236}{0.161} \\

    \bottomrule
  \end{tabular}
\end{table}

\paragraph{Annealing schedule without resampling.}
Table~\ref{tab:annealing-ablation-no-resampling} compares different annealing schedules for A-TGD without resampling. For inpainting, the choice of annealing schedule has only a small effect on reconstruction quality: uniform annealing gives the best SSIM and LPIPS, while the noise-dependent schedule with $\alpha=0.5$ gives the best PSNR. In contrast, for phase retrieval, annealing does not improve performance. The no-annealing variant achieves the best PSNR, SSIM, and LPIPS, while stronger noise-dependent annealing progressively degrades reconstruction quality.

\begin{table}[t]
  \caption{Annealing-schedule ablation for A-TGD on FFHQ with always resampling.}
  \label{tab:annealing-ablation-always-resampling}
  \centering
  \small
  \begin{tabular}{llcccc}
    \toprule
    Task & Annealing schedule & Time/img (s) $\downarrow$ & PSNR $\uparrow$ & SSIM $\uparrow$ & LPIPS $\downarrow$ \\
    \midrule

    \multirow{5}{*}{Inpainting}
      & Uniform annealing & 3.19 & \ms{25.01}{2.34} & \bestms{0.849}{0.025} & \bestms{0.124}{0.028} \\
      & Noise-dependent, $\alpha=0.5$ & 3.19 & \ms{24.94}{2.21} & \ms{0.847}{0.024} & \ms{0.126}{0.028} \\
      & Noise-dependent, $\alpha=1$ & 3.19 & \ms{24.90}{2.37} & \ms{0.848}{0.024} & \ms{0.125}{0.028} \\
      & Noise-dependent, $\alpha=2$ & 3.19 & \ms{24.89}{2.38} & \ms{0.847}{0.024} & \ms{0.126}{0.029} \\
      & Noise-dependent, $\alpha=4$ & 3.18 & \bestms{25.05}{2.26} & \ms{0.848}{0.024} & \ms{0.125}{0.028} \\

    \midrule

    \multirow{5}{*}{Phase Retrieval}
      & Uniform annealing & 3.25 & \ms{21.66}{7.27} & \ms{0.654}{0.210} & \ms{0.318}{0.193} \\
      & Noise-dependent, $\alpha=0.5$ & 3.25 & \bestms{24.77}{7.40} & \bestms{0.735}{0.198} & \bestms{0.248}{0.175} \\
      & Noise-dependent, $\alpha=1$ & 3.24 & \ms{23.92}{7.79} & \ms{0.712}{0.208} & \ms{0.269}{0.186} \\
      & Noise-dependent, $\alpha=2$ & 3.24 & \ms{24.18}{7.52} & \ms{0.719}{0.204} & \ms{0.267}{0.187} \\
      & Noise-dependent, $\alpha=4$ & 3.24 & \ms{22.25}{7.86} & \ms{0.668}{0.220} & \ms{0.309}{0.206} \\

    \bottomrule
  \end{tabular}
\end{table}

\paragraph{Annealing schedule with always resampling.}
Table~\ref{tab:annealing-ablation-always-resampling} reports the same annealing-schedule ablation when resampling is performed at every annealing step. Compared to the no-resampling setting in Table~\ref{tab:annealing-ablation-no-resampling}, always resampling gives small improvements for inpainting. For phase retrieval, always resampling substantially degrades reconstruction quality across all annealing schedules: even the best resampling variant remains worse than the corresponding no-resampling result, and uniform annealing is particularly unstable. These results suggest that frequent resampling can remove useful particle diversity in harder inverse problems.

\FloatBarrier



\paragraph{Paired significance analysis.}
Table~\ref{tab:paired-tests} reports paired comparisons of A-TGD
with DAPS ($N=4$) using measurements on the same 100 validation
images for each task and dataset. We report the mean paired
advantage $\Delta$, defined as A-TGD minus DAPS for PSNR and SSIM,
and DAPS minus A-TGD for LPIPS, so positive values always favour
A-TGD. We obtain 95\% confidence intervals using 100,000 paired
bootstrap replicates and compute two-sided paired sign-flip
permutation tests. The resulting $p$-values are Holm-corrected
across the three metrics separately for each task and dataset.

\begin{table}[t]
\centering
\small
\setlength{\tabcolsep}{4pt}
\caption{Paired comparisons of A-TGD with DAPS ($N=4$), using
the evaluation results reported during rebuttal on 100 images per
setting. Positive $\Delta$ favours A-TGD. Brackets contain 95\%
paired-bootstrap confidence intervals. The permutation-test
$p$-values are Holm-corrected across the three metrics within
each setting. ``A-TGD wins'' reports the percentage of images
on which A-TGD achieves the better metric value.}
\label{tab:paired-tests}
\begin{tabular}{lccccc}
\toprule
Metric & A-TGD & DAPS-4 & $\Delta$ [95\% CI]
& $p_{\mathrm{Holm}}$ & A-TGD wins \\
\midrule
\multicolumn{6}{l}{\textbf{FFHQ inpainting}} \\
PSNR $\uparrow$
& 25.1436 & 24.5204
& $0.6232\ [0.3783,\,0.8691]$
& $3.0\times10^{-5}$ & 63\% \\
LPIPS $\downarrow$
& 0.1224 & 0.1568
& $0.0344\ [0.0317,\,0.0371]$
& $3.0\times10^{-5}$ & 100\% \\
SSIM $\uparrow$
& 0.8525 & 0.7899
& $0.0626\ [0.0589,\,0.0663]$
& $3.0\times10^{-5}$ & 100\% \\
\midrule
\multicolumn{6}{l}{\textbf{ImageNet inpainting}} \\
PSNR $\uparrow$
& 20.8902 & 20.7031
& $0.1872\ [-0.1328,\,0.5100]$
& 0.261 & 53\% \\
LPIPS $\downarrow$
& 0.2041 & 0.2273
& $0.0232\ [0.0172,\,0.0292]$
& $3.0\times10^{-5}$ & 73\% \\
SSIM $\uparrow$
& 0.7809 & 0.7402
& $0.0407\ [0.0354,\,0.0461]$
& $3.0\times10^{-5}$ & 96\% \\
\midrule
\multicolumn{6}{l}{\textbf{FFHQ phase retrieval}} \\
PSNR $\uparrow$
& 29.4316 & 27.8105
& $1.6211\ [0.9524,\,2.3649]$
& $3.0\times10^{-5}$ & 83\% \\
LPIPS $\downarrow$
& 0.1715 & 0.2068
& $0.0352\ [0.0181,\,0.0533]$
& $3.0\times10^{-5}$ & 92\% \\
SSIM $\uparrow$
& 0.8111 & 0.7654
& $0.0457\ [0.0277,\,0.0655]$
& $3.0\times10^{-5}$ & 96\% \\
\midrule
\multicolumn{6}{l}{\textbf{ImageNet phase retrieval}} \\
PSNR $\uparrow$
& 22.179 & 22.674
& $-0.495\ [-2.286,\,1.302]$
& 1.000 & 42\% \\
LPIPS $\downarrow$
& 0.3256 & 0.3170
& $-0.0086\ [-0.0553,\,0.0382]$
& 1.000 & 44\% \\
SSIM $\uparrow$
& 0.6069 & 0.6081
& $-0.0012\ [-0.0543,\,0.0527]$
& 1.000 & 46\% \\
\bottomrule
\end{tabular}
\end{table}

The paired analysis finds statistically significant advantages for A-TGD across all three metrics on both FFHQ tasks, and for LPIPS and SSIM on ImageNet inpainting. The PSNR difference on ImageNet inpainting is not statistically significant ($p_{\mathrm{Holm}}=0.261$), and none of the three differences on ImageNet phase retrieval is statistically significant ($p_{\mathrm{Holm}}=1.000$).

\section{Broader Impacts}
\label{app:broader}

TGD and A-TGD are intended as methodological tools for training-free conditional diffusion and inverse-problem reconstruction. Potential positive impacts include more efficient use of pretrained diffusion priors for scientific imaging, restoration, and other inverse problems where paired task-specific training data are unavailable or costly. A-TGD may also reduce computational cost relative to independent multi-trajectory baselines by using multiple particles early and pruning to a single trajectory, improving speed-quality tradeoffs under fixed wall-clock budgets.

The same capabilities also have potential negative impacts. Improvements in image reconstruction and inpainting could be misused for image manipulation, misleading visual content generation, or reconstruction of sensitive visual information from incomplete measurements. These risks are partially mitigated by the scope of this work: we do not release a new pretrained generative model or dataset, and our experiments are limited to standard benchmark datasets and inverse-problem settings. Nevertheless, downstream deployments should consider data provenance, consent, privacy, and misuse monitoring, especially in applications involving people or sensitive imagery.

\section{Further Image Examples}
\label{app:further-image-examples}

We provide additional qualitative examples for each image inverse-problem setting. For each task and
dataset, we visualize examples with the largest absolute LPIPS difference between A-TGD and
DAPS \((N=4)\). These examples are intended to highlight cases where the methods differ most in
perceptual reconstruction quality and are taken from the same evaluation runs used for the matched-runtime results in
Table~\ref{tab:image-results}.

\begin{figure}[H]
    \centering
    \includegraphics[width=\linewidth]{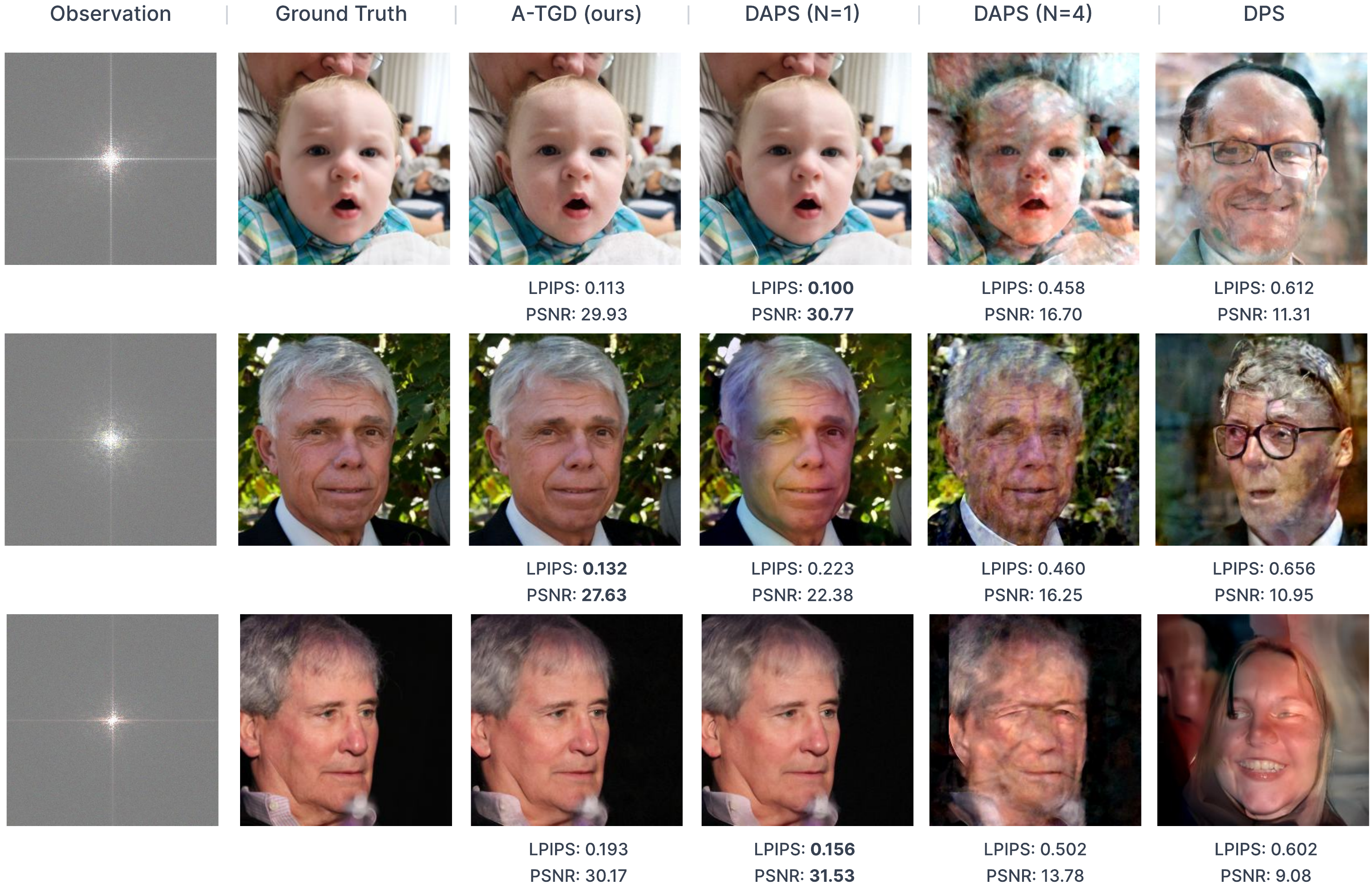}
    \caption{
    FFHQ phase retrieval examples.
    }
    \label{fig:examples-ffhq-pr}
\end{figure}

\begin{figure}[t]
    \centering
    \includegraphics[width=\linewidth]{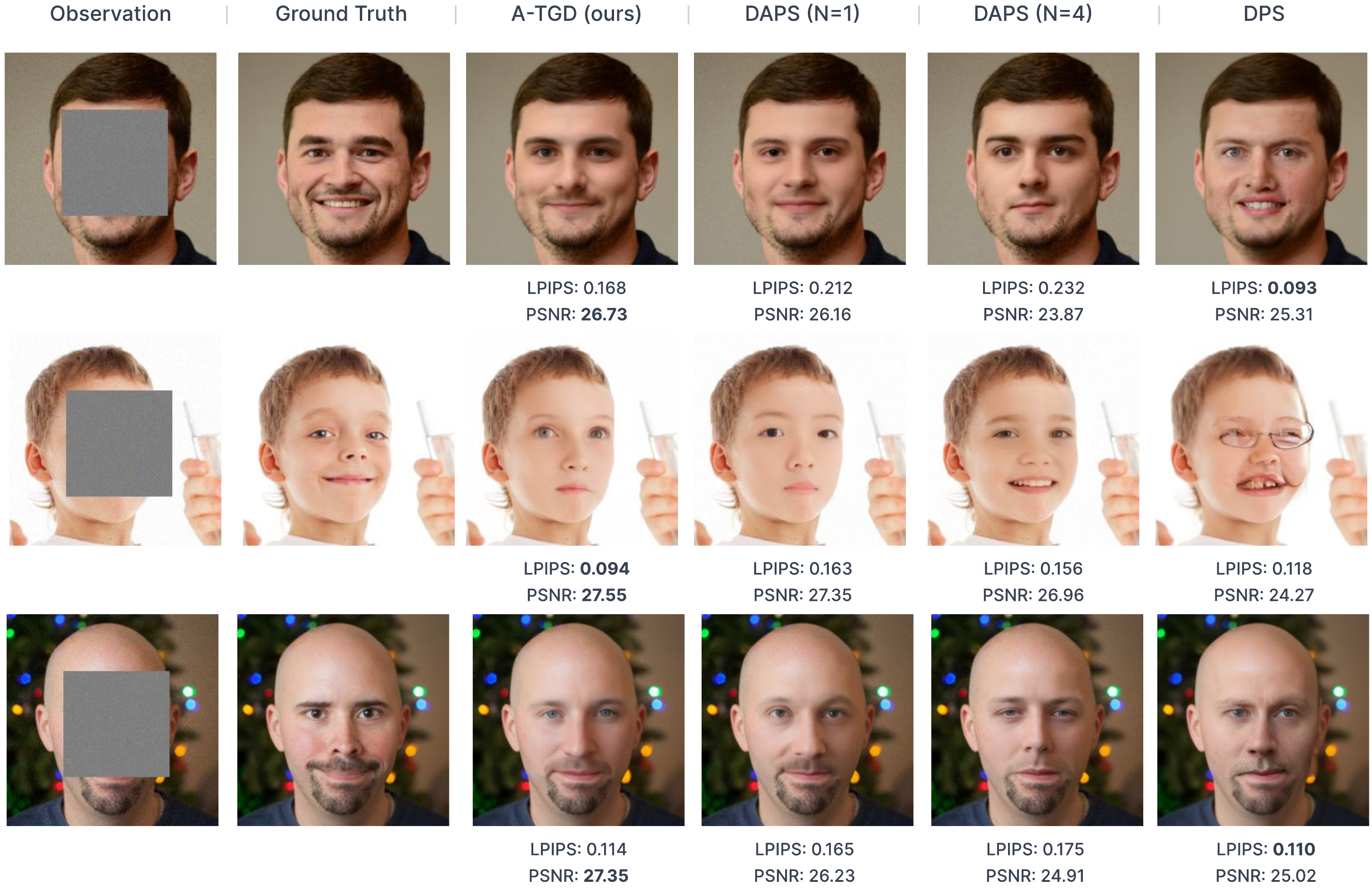}
    \caption{
    FFHQ inpainting examples.
    }
    \label{fig:examples-ffhq-inpainting}
\end{figure}

\begin{figure}[t]
    \centering
    \includegraphics[width=\linewidth]{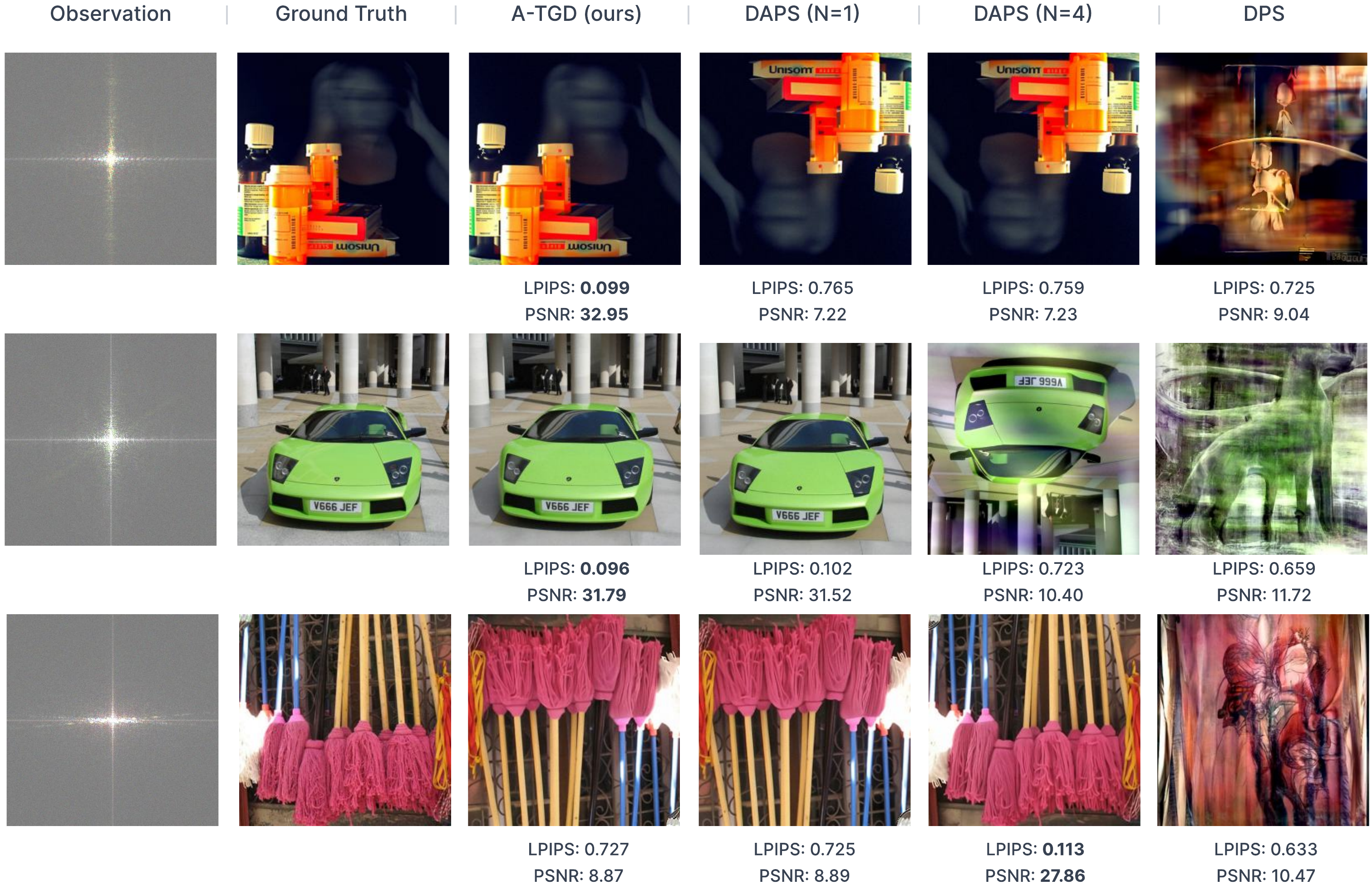}
    \caption{
    ImageNet phase retrieval examples.
    }
    \label{fig:examples-imagenet-pr}
\end{figure}

\begin{figure}[t]
    \centering
    \includegraphics[width=\linewidth]{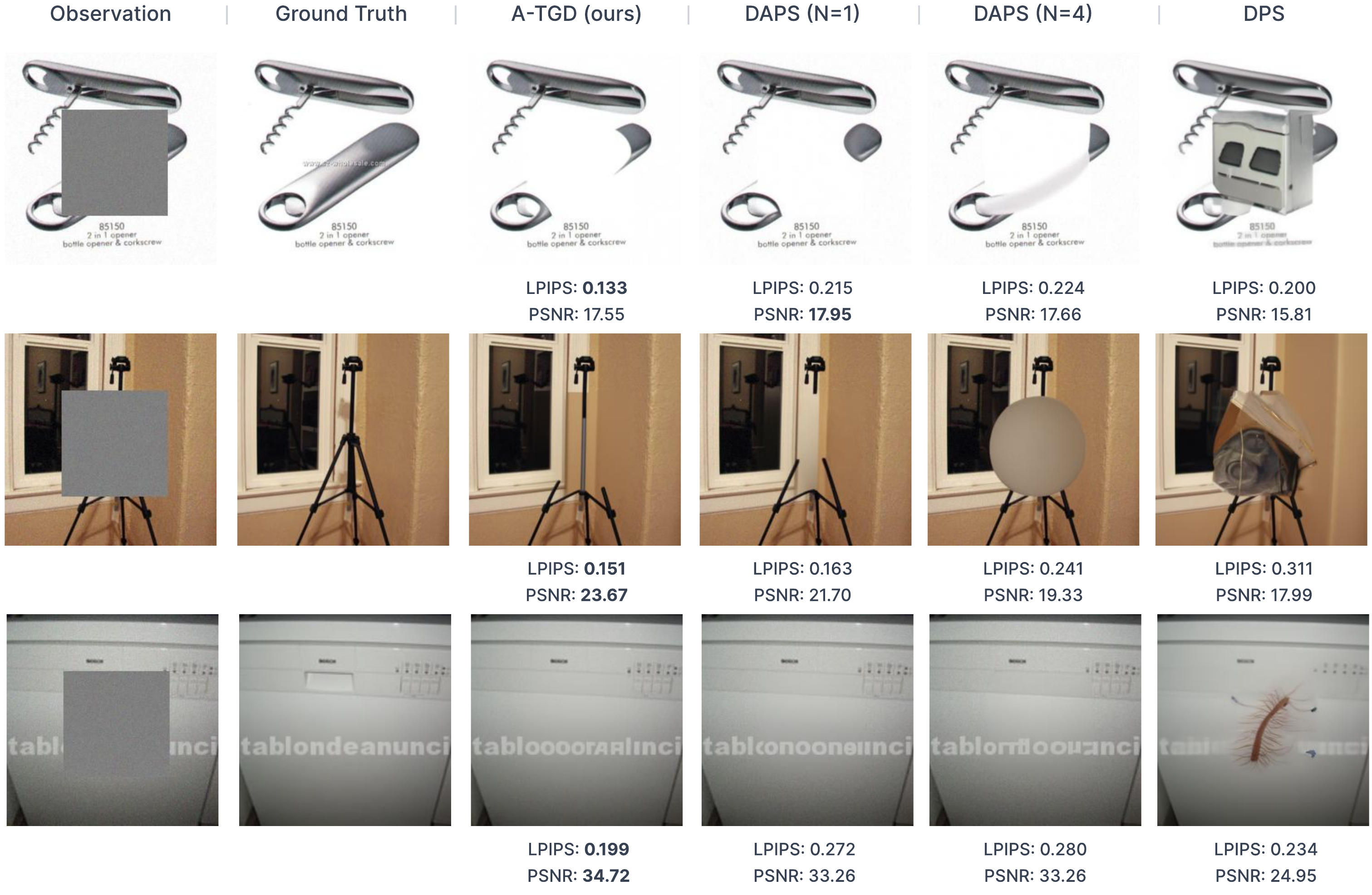}
    \caption{
    ImageNet inpainting examples.
    }
    \label{fig:examples-imagenet-inpainting}
\end{figure}


\newpage
\FloatBarrier

\end{document}